\def\isarxiv{1} 

\ifdefined\isarxiv
\documentclass[11pt]{article}

\usepackage[numbers]{natbib}

\else
\documentclass{article}
\usepackage{algorithm}
\usepackage{microtype}
\usepackage{graphicx}
\usepackage{subfigure}
\usepackage{booktabs}
\usepackage{hyperref}
\usepackage{icml2025}

\usepackage{amsmath}
\usepackage{amssymb}
\usepackage{mathtools}
\usepackage{amsthm}
\usepackage{algorithm}
\usepackage{algpseudocode}
\fi

\ifdefined\isarxiv
\usepackage{amsmath}
\usepackage{amsthm}
\usepackage{amssymb}
\usepackage{algorithm}
\usepackage{subfig}
\usepackage{algpseudocode}
\usepackage{graphicx}
\usepackage{grffile}
\usepackage{wrapfig,epsfig}
\usepackage{url}
\usepackage{xcolor}
\usepackage{epstopdf}

\usepackage{bbm}
\usepackage{dsfont}
\fi

\allowdisplaybreaks

\ifdefined\isarxiv

\usepackage{tikz}
\usepackage{hyperref}  
\hypersetup{colorlinks=true,citecolor=blue,linkcolor=blue} 
\usetikzlibrary{arrows}
\usepackage[margin=1in]{geometry}

\else

\fi

\graphicspath{{./figs/}}

\theoremstyle{plain}
\newtheorem{theorem}{Theorem}[section]
\newtheorem{lemma}[theorem]{Lemma}
\newtheorem{definition}[theorem]{Definition}

\newtheorem{proposition}[theorem]{Proposition}

\newtheorem{fact}[theorem]{Fact}
\newtheorem{remark}[theorem]{Remark}

\newtheorem{example}[theorem]{Example}

\newcommand{\ov}{\overline}

\newcommand{\NPhard}{\mathsf{NP}\text{-}\mathsf{hard}}

\DeclareMathOperator{\Diag}{Diag}

\makeatletter
\newcommand*{\RN}[1]{\expandafter\@slowromancap\romannumeral #1@}
\makeatother

\usepackage{lineno}

\ifdefined\isarxiv
\else
\icmltitlerunning{Dissecting Submission Limit in Desk-Rejections: A Mathematical Analysis of Fairness in AI Conference Policies}
\fi
\begin{document}

\ifdefined\isarxiv

\date{}

\title{Dissecting Submission Limit in Desk-Rejections: A Mathematical Analysis of Fairness in AI Conference Policies}
\author{
Yuefan Cao\thanks{\texttt{ ralph1997off@gmail.com}. Zhejiang University.}
\and
Xiaoyu Li\thanks{\texttt{
xiaoyu.li2@student.unsw.edu.au}. University of New South Wales.}
\and
Yingyu Liang\thanks{\texttt{
yingyul@hku.hk}. The University of Hong Kong. \texttt{
yliang@cs.wisc.edu}. University of Wisconsin-Madison.} 
\and
Zhizhou Sha\thanks{\texttt{
shazz20@mails.tsinghua.edu.cn}. Tsinghua University.}
\and
Zhenmei Shi\thanks{\texttt{
zhmeishi@cs.wisc.edu}. University of Wisconsin-Madison.}
\and
Zhao Song\thanks{\texttt{ magic.linuxkde@gmail.com}. The Simons Institute for the Theory of Computing at UC Berkeley.}
\and
Jiahao Zhang\thanks{\texttt{ ml.jiahaozhang02@gmail.com}. Independent Researcher.}
}

\else


\twocolumn[

\icmltitle{Dissecting Submission Limit in Desk-Rejections: A Mathematical Analysis\\ of Fairness in AI Conference Policies}


\icmlsetsymbol{equal}{*}

\begin{icmlauthorlist}
\icmlauthor{Aeiau Zzzz}{equal,to}
\icmlauthor{Bauiu C.~Yyyy}{equal,to,goo}
\icmlauthor{Cieua Vvvvv}{goo}
\icmlauthor{Iaesut Saoeu}{ed}
\icmlauthor{Fiuea Rrrr}{to}
\icmlauthor{Tateu H.~Yasehe}{ed,to,goo}
\icmlauthor{Aaoeu Iasoh}{goo}
\icmlauthor{Buiui Eueu}{ed}
\icmlauthor{Aeuia Zzzz}{ed}
\icmlauthor{Bieea C.~Yyyy}{to,goo}
\icmlauthor{Teoau Xxxx}{ed}\label{eq:335_2}
\icmlauthor{Eee Pppp}{ed}
\end{icmlauthorlist}

\icmlaffiliation{to}{Department of Computation, University of Torontoland, Torontoland, Canada}
\icmlaffiliation{goo}{Googol ShallowMind, New London, Michigan, USA}
\icmlaffiliation{ed}{School of Computation, University of Edenborrow, Edenborrow, United Kingdom}

\icmlcorrespondingauthor{Cieua Vvvvv}{c.vvvvv@googol.com}
\icmlcorrespondingauthor{Eee Pppp}{ep@eden.co.uk}

\icmlkeywords{Machine Learning, ICML}

\vskip 0.3in
]

\printAffiliationsAndNotice{\icmlEqualContribution} 
\fi

\ifdefined\isarxiv
\begin{titlepage}
  \maketitle
  \begin{abstract}

As AI research surges in both impact and volume, conferences have imposed submission limits to maintain paper quality and alleviate organizational pressure. 
In this work, we examine the fairness of desk-rejection systems under submission limits and reveal that existing practices can result in substantial inequities. Specifically, we formally define the paper submission limit problem and identify a critical dilemma: when the number of authors exceeds three, it becomes impossible to reject papers solely based on excessive submissions without negatively impacting innocent authors.
Thus, this issue may unfairly affect early-career researchers, as their submissions may be penalized due to co-authors with significantly higher submission counts, while senior researchers with numerous papers face minimal consequences. To address this, we propose an optimization-based fairness-aware desk-rejection mechanism and formally define two fairness metrics: individual fairness and group fairness. We prove that optimizing individual fairness is NP-hard, whereas group fairness can be efficiently optimized via linear programming. Through case studies, we demonstrate that our proposed system ensures greater equity than existing methods, including those used in CVPR 2025, offering a more socially just approach to managing excessive submissions in AI conferences.

  \end{abstract}
  \thispagestyle{empty}
\end{titlepage}

{\hypersetup{linkcolor=black}
\tableofcontents
}
\newpage

\else

\begin{abstract}

As AI research surges in both impact and volume, conferences have imposed submission limits to maintain paper quality and alleviate organizational pressure. 
In this work, we examine the fairness of desk-rejection systems under submission limits and reveal that existing practices can result in substantial inequities. Specifically, we formally define the paper submission limit problem and identify a critical dilemma: when the number of authors exceeds three, it becomes impossible to reject papers solely based on excessive submissions without negatively impacting innocent authors.
Thus, this issue may unfairly affect early-career researchers, as their submissions may be penalized due to co-authors with significantly higher submission counts, while senior researchers with numerous papers face minimal consequences. To address this, we propose an optimization-based fairness-aware desk-rejection mechanism and formally define two fairness metrics: individual fairness and group fairness. We prove that optimizing individual fairness is NP-hard, whereas group fairness can be efficiently optimized via linear programming. Through case studies, we demonstrate that our proposed system ensures greater equity than existing methods, including those used in CVPR 2025, offering a more socially just approach to managing excessive submissions in AI conferences.

\end{abstract}

\fi

\section{Introduction}

\begin{table}[!ht]\caption{ 
In this table, we summarize the submission limits of top conferences in recent years. For details of each conference website, we refer the readers to Section~\ref{app:sec:conference_links} in the Appendix. 
Some conferences (CVPR, ICCV, WSDM, KDD) employ a conventional desk-reject algorithm (Algorithm~\ref{alg:conventional_desk_reject_algo}), where papers are desk-rejected once an author has registered more than $x$ (the submission limit) papers.
}  \label{tab:conference_submission_limit}
\begin{center}
\begin{tabular}{ |c|c|c|c|c| } 
 \hline
 {\bf Conference Name} & {\bf Year} & {\bf Upper Bound} \\ \hline
 CVPR & 2025 & 25 \\ \hline
 CVPR & 2024 & N/A \\ \hline
 ICCV & 2025 & 25 \\ \hline
 ICCV & 2023 & N/A \\ \hline
 AAAI & 2023-2025 & 10 \\ \hline
 AAAI & 2022 & N/A \\ \hline
 WSDM & 2021-2025 & 10 \\ \hline
 WSDM & 2020 & N/A \\ \hline
 IJCAI & 2021-2025 & 8 \\ \hline
 IJCAI & 2020 & 6 \\ \hline
 IJCAI & 2018-2019 & 10 \\ \hline
 IJCAI & 2017 & N/A \\ \hline
 KDD & 2024-2025 & 7 \\ \hline
 KDD & 2023 & N/A \\ \hline
\end{tabular}
\end{center}
\end{table}

We are living in an era shaped by the unprecedented advancements of Artificial Intelligence (AI), where transformative breakthroughs have emerged across various domains in just a few years. A key driving force behind AI's rapid progress is the prevalence of top conferences held frequently throughout the year, offering dynamic platforms to present many of the field’s most influential papers. For example, ResNet~\cite{hzrs16}, a foundational milestone in deep learning with over 250,000 citations, was first introduced at CVPR 2016. Similarly, the Transformer architecture~\cite{vsp+17}, the backbone of modern large language models, emerged at NeurIPS 2017. More recently, diffusion models~\cite{hja20}, which represent the state-of-the-art in image generation, were presented at NeurIPS 2020, while CLIP~\cite{rkh+21}, a leading model for image-text pretraining, was showcased at ICML 2021. These groundbreaking contributions from top-tier conferences have significantly accelerated the advancement of AI, enriching both theoretical insights and practical applications.

As AI continues to expand its applications and capabilities in real-world domains such as dialogue systems~\cite{szk+22,aaa+23,a24}, image generation~\cite{hja20,sme20,wsd+24,wcz+23,wxz+24}, and video generation~\cite{hsg+22,bdk+23}, its immense potential for commercialization has raised growing enthusiasm in AI research. This enthusiasm has led to a rapid, rocket-like increase in the number of AI-related papers in 2024, as witnessed by recent studies~\cite{stanford_ai_index}. A direct consequence of this surge is the significant rise in submissions to AI conferences, which has placed a heavy burden on program committees tasked with selecting papers for acceptance. To address these challenges and maintain the quality of accepted papers, many leading conferences have introduced submission limits per author. In 2025, a wide range of leading AI conferences, including CVPR, ICCV, AAAI, WSDM, IJCAI, and KDD, have introduced submission limits per author in their guidelines, ranging from a maximum of $x=7$ to $x=25$. Table~\ref{tab:conference_submission_limit} provides an overview of these submission limits across major AI conferences.

However, such a desk-rejection mechanism may result in unintended negative societal impacts due to the Matthew effect in the research community~\cite{bvr18}, as illustrated in Figure~\ref{fig:matthew_effect}. Recent research has shown that the impact of a setback (e.g., a paper rejection) is often much greater for early-career researchers than for senior researchers~\cite{wjw19,scl23}, which shows that the effect of a desk rejection can vary significantly depending on the author’s career stage. For instance, as illustrated in Figure~\ref{fig:unfairness_of_desk_rejection}, consider the case of a young student submitting their only draft to the conference, co-authored with a renowned researcher who submits numerous papers annually. If the paper is desk-rejected due to exceeding submission limits, the senior researcher might view this as a neglectable inconvenience. In contrast, the rejection could have severe consequences for the student, as the paper might be crucial for applying to graduate programs, securing employment, or forming a chapter of their thesis. This disparity in the impact of desk rejections may worsen the Matthew effect in the AI community by disproportionately disadvantaging researchers with only one or two submitted papers, while having little effect on prolific senior researchers. Such outcomes raise important concerns about fairness and equity in current desk-rejection policies.

\begin{figure}[!ht]
    \centering
    \includegraphics[width=0.95\linewidth]{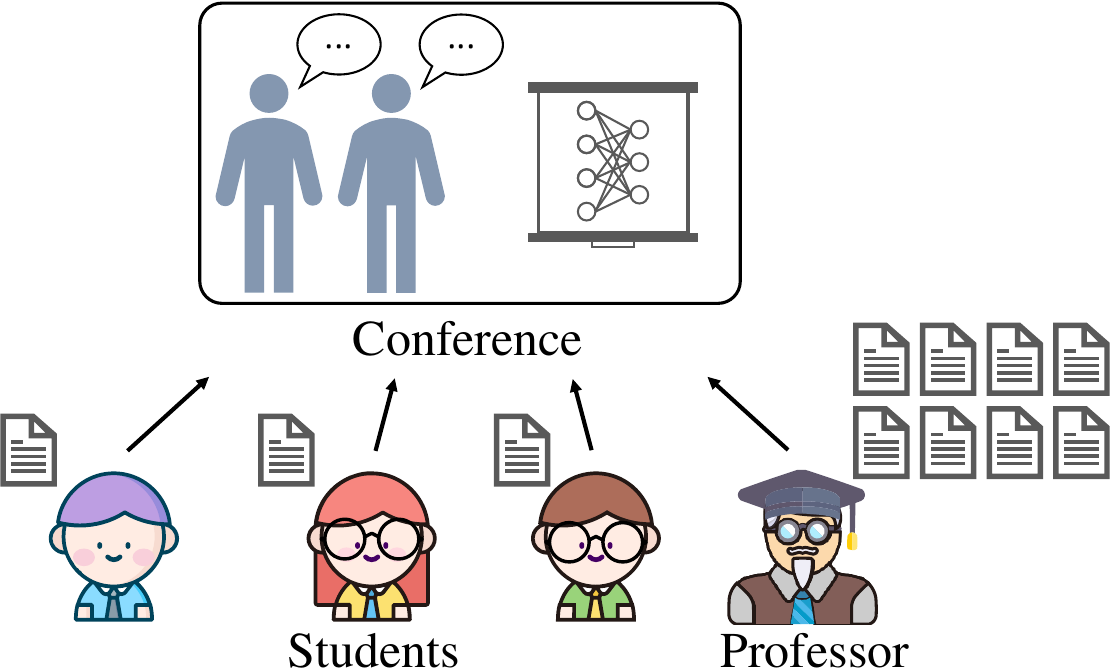}
    \caption{
    The Matthew Effect in the AI community. This figure illustrates the worsening Matthew Effect in the AI community, where senior researchers tend to have a significantly higher number of submissions, while junior researchers have relatively few. 
    }
    \label{fig:matthew_effect}
\end{figure}

In response to the challenges posed by paper-limit-based desk-rejection systems, this work investigates an important and practical problem: ensuring fairness in desk-rejection systems for AI conferences under submission limits. As illustrated in Figure~\ref{fig:our_research}, our goal is to design a fair desk-rejection system that prioritizes rejecting submissions from authors with many papers while protecting those with fewer submissions, particularly early-career researchers. Our key contributions are as follows:

\begin{itemize}
\item We formally define the paper submission limit problem in desk-rejection systems and prove that an ideal system that rejects papers solely based on each author's excessive submissions is mathematically impossible when there are more than three authors.
\item We introduce two fairness metrics: individual fairness and group fairness. We formulate the fairness-aware paper submission limit problem as an integer programming problem. Specifically, we formally prove that optimizing individual fairness is NP-hard, while the group fairness optimization problem can be solved efficiently using any off-the-shelf linear programming solver.
\item Through case studies, we demonstrate that our proposed system achieves greater fairness compared to existing approaches used in top AI conferences such as CVPR 2025, promoting social justice and fostering a more inclusive ML research community.
\end{itemize}

\paragraph{Roadmap.} Our paper is organized as follows: In Section~\ref{sec:related_works}, we review related literature. In Section~\ref{sec:preliminary}, we present the key definition of the paper submission limit problem. In Section~\ref{sec:dr_dilemma}, we show that no algorithm can reach the ideal desk-rejection system without unfair collective punishments. In Section~\ref{sec:fair}, we present our new fairness-aware desk-rejection system. In Section~\ref{sec:case_study}, we show by case studies that our system is better than existing systems. In Section~\ref{sec:conclusion}, we present our conclusions.

\begin{figure*}[!ht]
    \centering
    \includegraphics[width=0.85\linewidth]{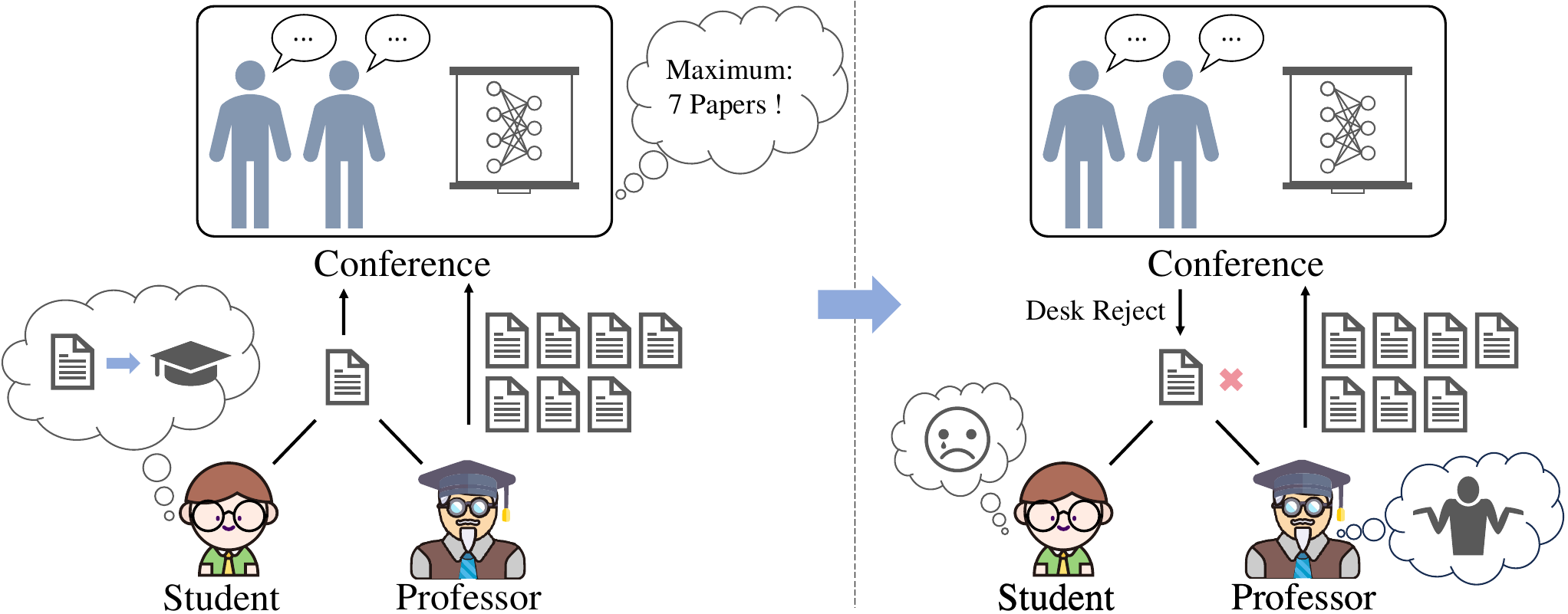}
    \caption{ The unfairness of desk rejection based on submission limits. {\bf Left: A careless mistake.} In this scenario, a young student submits the only paper, co-authored with a professor who submits numerous papers, and carelessly exceeds the submission limit. The paper, which may aim to apply to graduate programs, secure employment, or form a chapter of the thesis, is very important for the student but may not be for the professor. {\bf Right: The desk rejection.} If the paper is desk-rejected due to submission limits, it poses a minor inconvenience to the professor, and the professor can shrug about it due to his remaining papers. However, it could have severe consequences for the students, as the paper is crucial for the student's future plans.
    }
    \label{fig:unfairness_of_desk_rejection}
\end{figure*}
\section{Related Works} \label{sec:related_works}

\subsection{Desk Reject Mechanism}

A wide range of desk-rejection mechanisms have been developed to reduce the human effort involved in the peer-review process~\cite{as21}. One of the most widely adopted desk-rejection rules is rejecting papers that violate anonymity requirements~\cite{jawd02, t18}. This rule is crucial for maintaining unbiased evaluations of researchers from diverse institutions and career levels while preventing conflicts of interest. Another common mechanism addresses duplicate and dual submissions~\cite{s03, l13}, alleviating the duplication of reviewer efforts across multiple venues and upholding ethical publication standards. Additionally, plagiarism~\cite{kc23, er23} is a major concern in desk rejections at AI conferences, as it undermines the integrity of the academic community, violates intellectual property rights, and compromises the originality and credibility of research. In response to the growing number of submissions to AI conferences, new types of desk-rejection rules have recently emerged~\cite{lnz+24}. For example, IJCAI 2020 and NeurIPS 2020 implemented a fast desk-rejection mechanism, allowing area chairs to reject papers based on a quick review of the abstract and main content to manage the review workload. However, this approach introduced noise and sometimes resulted in the rejection of generally good papers, leading to its reduced prevalence compared to more systematic mechanisms like enforcing submission limits, which is the main focus of this paper. To the best of our knowledge, limited literature has explored these emerging desk-rejection techniques, and our work is among the first to formally study the desk-rejection mechanism based on maximum submission limits.

\subsection{The Competitive Race in AI Publication}

Due to the rapid increase in submissions to AI conferences in recent years~\cite{stanford_ai_index}, concerns about the intense competition in these conferences are growing. As Bengio Yoshua noted~\cite{b20blog}: ``It is more competitive, everything is happening fast and putting a lot of pressure on everyone. The field has grown exponentially in size. Students are more protective of their ideas and in a hurry to put them out, by fear that someone else would be working on the same thing elsewhere, and in general, a PhD ends up with at least 50\% more papers than what I gather it was 20 or 30 years ago.'' Consequently, paper acceptance has become increasingly critical in AI job applications~\cite{a22, bbm+24}, as having more papers is now the norm. Therefore, it is crucial to establish fair and practical guidelines for desk rejections~\cite{takb18}, ensuring that every group of authors is treated equitably in AI conferences.

\subsection{Fairness System Design}

Fairness~\cite{f12, mms+21} is a key principle of social justice, reflecting the absence of bias toward individuals or groups based on inherent characteristics. Due to its profound societal impact, fairness has become an essential consideration in the design of algorithms across various computer systems that interact with human factors. In recommender systems, fairness can manifest in various forms, such as item fairness~\cite{zfh+21, glg+21}, which ensures that items from different categories or with varying levels of prior exposure are recommended equitably, and user fairness~\cite{lcf+21, lcx+21}, which guarantees that all users, regardless of their backgrounds or preferences, have equal opportunities to access relevant content. These fairness measures help balance opportunities for both users and retailers, fostering equity in the recommendation process. In candidate selection systems~\cite{g93, whz20}, fairness ensures that all candidates are evaluated solely on merit, independent of factors such as race, gender, or socio-economic background, promoting equality and ensuring that the selection processes are inclusive. In information access systems~\cite{edb+22}, including job search~\cite{wmm+22} and music discovery~\cite{mrp+21}, fairness guarantees that all individuals can access the information they need without discrimination, ensuring equal opportunities for users to make informed decisions. Similarly, in dialog systems~\cite{gya22,grb+24}, fairness ensures that language models avoid generating biased text or making inappropriate word-context associations related to social groups, supporting equitable and respectful interactions. Moreover, recent research has investigated group fairness in peer review processes for AI conferences, highlighting the importance of equitable evaluation for submissions~\cite{ams23}.
Despite the widespread focus on fairness in algorithmic design, the fairness of desk-rejection mechanisms remains an open question and serves as the primary motivation for this paper.

\section{Preliminary} \label{sec:preliminary}
In this section, we first introduce the notations in Section~\ref{sec:notations}. Then, we present the general problem formulation in Section~\ref{sec:problem_formulation}.

\subsection{Notations}\label{sec:notations}
For any positive integer $n$, we use $[n]$ to denote the set $\{1, 2, \ldots, n\}$. We use $\mathbb{N}_+$ to represent the set of all positive integers. For two sets $\mathcal{B}$ and $\mathcal{C}$, we denote the set difference as $\mathcal{B} \setminus \mathcal{C}:=\{x\in \mathcal{B}:x\notin\mathcal{C}\}$. For a vector $x \in \mathbb{R}^d$, $\Diag(d)$ denotes a diagonal matrix $X \in \mathbb{R}^{d \times d}$, where the diagonal entries satisfy $X_{i,i} = x_i$ for all $i \in [d]$, and all off-diagonal entries are zero. We use $\mathbf{1}_n$ to denote an $n$-dimensional column vector with all entries equal to one.

\begin{figure}[!ht]
    \centering
    \includegraphics[width=0.95\linewidth]{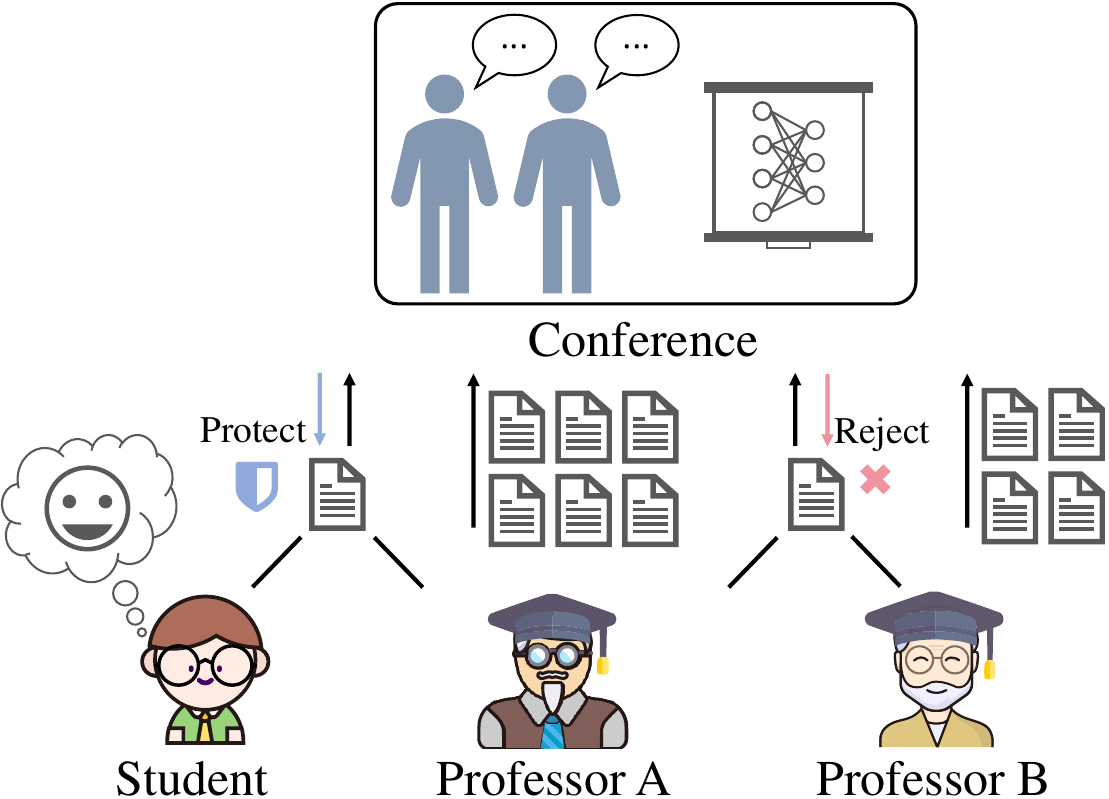}
    \caption{
    Our research objective. This figure presents the goal of our study: creating a more equitable desk-rejection system. Consider Professor A, who has carelessly submitted numerous papers exceeding the submission limit, collaborating with another senior researcher (Professor B) with many submissions, and a young student with only one paper. Our proposed system prioritizes desk-rejecting papers from authors with a large number of submissions first, thereby increasing the student’s chances of having their paper accepted. This approach aims to mitigate the disparity in the impact of desk rejections and promote fairness.
    }
    \label{fig:our_research}
\end{figure}

\subsection{Problem Formulation} \label{sec:problem_formulation}

In this section, we further introduce the actual problem we will investigate in this paper, where we begin with introducing the definition for three kinds of authors that will appear later in our discussion.

\begin{definition}[Submission Limit Problem]\label{def:submit_limit_problem}
    Let $\mathcal{A} = \{a_1, a_2, \dots, a_n\}$ denote the set of $n$ authors, and let $\mathcal{P} = \{p_1, p_2, \dots, p_m\}$ denote the set of $m$ papers. Each author $a_i \in \mathcal{A}$ has a subset of papers $P_i \subseteq \mathcal{P}$, and each paper $p_j \in \mathcal{P}$ is authored by a subset of authors $A_j \subseteq \mathcal{A}$. For each author, $a_i \in \mathcal{A}$, let $C_i$ denote the set of all coauthors of $a_i$ and let $x \in \mathbb{N}_+$ denote the maximum number of papers each author can submit. 

    The goal is to find a subset $S \subseteq \mathcal{P}$ of papers (to keep) such that for every $a_i\in\mathcal{A},$
    \begin{align*}
        \underbrace{|\{p_j \in  S : a_i \in A_j\}|}_{\#\mathrm{remained~papers~of~author}~a_i} \leq x.  
    \end{align*}
    or equivalently find a subset $\ov{S} \subseteq \mathcal{P}$ of papers (to reject) such that for every $a_i \in \mathcal{A}$,
        \begin{align*}
        |P_i| - \underbrace{|\{j \in  \ov{S} : i \in A_j\}|}_{\#{\mathrm{rejected~papers~of~author}~}a_i} \leq x.  
    \end{align*}
\end{definition}

We now present several fundamental facts related to Definition~\ref{def:submit_limit_problem}, which can be easily verified through basic set theory. 
\begin{fact}
    For any author $a_i \in \mathcal{A}$ and paper $p_j \in \mathcal{P}$, $a_i \in A_j$ if and only if $p_j \in P_i$.
\end{fact}

\begin{fact}
    For each author $a_i \in \mathcal{A}$, the number of papers submitted by the author can be formulated as:
    \begin{align*}
        |P_i| = |\{p_j \in \mathcal{P} : a_i \in A_j\}|.
    \end{align*}
\end{fact}

\begin{fact}
    For each paper $j \in [m]$, the number of authors of this paper can be formulated as:
    \begin{align*}
        |A_j| = |\{a_i \in \mathcal{A} : p_j \in P_i\}|.
    \end{align*}
\end{fact}

\begin{fact}
    For each author $a_i \in \mathcal{A}$, the set of coauthors for author $a_i$ can be formulated as:
    \begin{align*}
        C_i = (\bigcup_{p_j \in P_i} A_j) \setminus \{ a_i \}.
    \end{align*}
\end{fact}

\section{The Desk Rejection Dilemma}\label{sec:dr_dilemma}

In this section, we define the concept of an ideal desk-rejection system in Section~\ref{sec:good_solution} and formally demonstrate in Section~\ref{sec:good_solution_hard} that no algorithm can achieve this ideal system.

\subsection{Ideal Desk-Rejection}\label{sec:good_solution}
An ideal desk-rejection system should avoid unfairly rejecting papers from authors who either comply with the submission limit or exceed it by only one or two papers. Otherwise, authors may face consequences due to co-authors with an excessively high number of submissions. This issue is particularly problematic for early-career researchers, as such collective penalties can have a significant negative impact on their careers. 

To address this, we formally define the criteria for an ideal desk-rejection outcome for the problem in Definition~\ref{def:submit_limit_problem}, where rejections are based solely on an author’s excessive submissions, without unfairly penalizing others.

\begin{definition}[Ideal desk-rejection] \label{def:good_solution}
An ideal solution for the submission limit problem in Definition~\ref{def:submit_limit_problem} is a paper subset $S\subseteq \mathcal{P}$ such that every author has exactly $\min\{x, |P_i|\}$ papers remaining after desk rejection. 

\end{definition}
\begin{remark}
    The ideal desk-rejection in Definition~\ref{def:good_solution} ensures that innocent authors with less than $x$ submissions will retain all their papers, and a non-compliant author $a_i$ with more than $x$ submissions will be desk-rejected exact $(|P_i|-x)$ papers.
\end{remark}
Thus, if there exists an algorithm that can reach the aforementioned ideal solution, we can ensure that no author is unfairly penalized due to their co-authors' submission behavior, achieving both fairness and individual accountability.
\subsection{Hardness of Ideal Desk-Rejection}\label{sec:good_solution_hard}

Unfortunately, we find that achieving an ideal desk-rejection system is fundamentally intractable. The main result regarding this hardness is presented in the following theorem:

\begin{theorem}[Hardness of Ideal Desk-Rejection]\label{thm:main_res_general}
Let $n = |\mathcal{A}|$ denote the number of authors in Definition~\ref{def:submit_limit_problem}. We can show that

\begin{itemize}
    \item {\bf Part 1:} For $n \le 2$, there always exists an algorithm that can achieve the ideal desk-rejection in Definition~\ref{def:good_solution}.
    \item {\bf Part 2:} For $n \geq 3$, there exists at least one problem instance where no algorithm can guarantee achieving the ideal desk-rejection in Definition~\ref{def:good_solution}.
\end{itemize}
\end{theorem}

\begin{proof} For {\bf Part 1}, the result follows directly from Lemma~\ref{lem:n_eq_1_positive_general} and Lemma~\ref{lem:n_eq_2_positive_general}. For {\bf Part 2}, the result is established using Lemma~\ref{lem:n_eq_3_negative} and Lemma~\ref{lem:n_geq_3_negative}. 
Detailed technical proofs for these lemmas are provided in Appendix~\ref{sec:dilemma_proof}.
\end{proof}

Therefore, since an ideal desk-rejection system is not achievable, it is inevitable that some authors may face excessive desk-rejections due to collective punishments. This challenge is particularly concerning for early-career researchers with only one or two submissions, motivating the need to seek an approximate solution that optimizes fairness in desk-rejection systems.

\section{Fairness-Aware Desk-Rejection}\label{sec:fair}

In this section, we first introduce two fairness metrics in Section~\ref{sec:fair_metric}, and then present the hardness result on minimizing one of them in Section~\ref{sec:good_solution_hard}. In Section~\ref{sec:fair_optim}, we show our optimization-based fairness-aware desk-rejection framework. 

\subsection{Fairness Metrics}\label{sec:fair_metric}

As discussed earlier, achieving an ideal desk-rejection system is practically infeasible, as unintended rejections due to collective punishments are unavoidable. To address this, we relax the ideal system into an approximate form, where some unfair desk-rejections are permitted, while these rejections should be proportional to each author's total number of submissions.

Specifically, we introduce a cost function for each author, which estimates the impact of desk-rejection on each author:

\begin{definition}[Cost Function]\label{def:cost}
Considering the submission limit problem in Definition~\ref{def:submit_limit_problem}, we define the cost function $c: [n] \times 2^{[m]} \to [0,1]$ for a specific author $a_i$ and a set of remaining paper $S$ as
\begin{align*}
    c(a_i, S) := \frac{|P_i| - |\{p_j \in S: a_i \in A_j\}|}{|P_i|}.
\end{align*}
\end{definition}

\begin{remark}
    The cost function $c(a_i,S)$ measures the proportion of papers authored by $a_i$ that are rejected, prioritizing fairness for early-career authors with fewer submissions and aiming to reduce setbacks for them.
\end{remark}

To further demonstrate how this author-wise cost function could benefit fairness, we present the following example:
\begin{example}
    Consider a submission limit problem with $x = 10$ and $n = 2$. Suppose author $a_1$ submits papers $p_1, p_2, \ldots, p_{11}$, and author $a_2$ submits only paper $p_{11}$. Rejecting paper $p_{11}$ (i.e., $S = \mathcal{P} \setminus \{p_{11}\}$) results in a cost of $c(a_1, S) = 1/11$ for $a_1$ but a cost of $c(a_2, S) = 1$ for $a_2$, which is unfair to $a_2$. On the other hand, if we reject paper $p_1$ (i.e., $S' = \mathcal{P} \setminus \{p_1\}$), the cost for $a_1$ remains $c(a_1, S') = 1/11$, while the cost for $a_2$ becomes $c(a_2, S') = 0$. This minimizes both the highest cost and the total cost. This example demonstrates that our cost function encourages rejecting papers from authors with many submissions while protecting authors with few submissions.
\end{example}

To ensure fair treatment for all authors and avoid imposing excessive setbacks on early-career researchers, we introduce two fairness metrics based on our cost function. These metrics are inspired by the principles of utilitarian social welfare and egalitarian social welfare~\cite{ams24}. We begin by defining individual fairness, which is a strict worst-case fairness metric that aligns with the egalitarian social welfare framework by estimating the individual cost among all authors.

\begin{definition}[Individual Fairness]\label{def:individual_fair}
Let $c: [n] \times 2^{[m]} \to [0,1] $ be the cost function defined in Definition~\ref{def:cost}.
We define function $\zeta_{\mathrm{ind}}: 2^{[m]} \to [0,1]$ to measure the individual fairness:
\begin{align*}
    & ~ \zeta_{\mathrm{ind}}(S) := \max_{i \in [n]} c(a_i,S).
\end{align*}
\end{definition}
Next, we present the concept of group fairness, which aligns with utilitarian social welfare and measures the total cost across all authors.

\begin{definition}[Group Fairness]\label{def:group_fair}
Let $c: [n] \times 2^{[m]} \to [0,1] $ be the cost function defined in Definition~\ref{def:cost}.
We define function $\zeta_{\mathrm{group}}: 2^{[m]} \to [0,1]$ to measure the group fairness:
\begin{align*}
    & ~ \zeta_{\mathrm{group}}(S) := \frac{1}{n}\sum_{i \in [n]} c(a_i,S).
\end{align*}
 \end{definition}

To show the relationship between these two fairness metrics, we have the following proposition:

\begin{proposition}[Relationship of Fairness Metrics, informal version of Proposition~\ref{lem:fair_metric_ineq_append} in Appendix~\ref{sec:fair_proof}]\label{lem:fair_metric_ineq}
    For any solution $S\subseteq \mathcal{P}$ to the submission limit problem in Definition~\ref{def:submit_limit_problem}, we have 
    \begin{align*}
        \zeta_{\mathrm{ind}}(S) \leq \zeta_{\mathrm{group}}(S).
    \end{align*}
\end{proposition}

\subsection{Hardness of Individual Fairness-Aware Submission Limit Problem}\label{sec:indi_fair_hard}

After presenting fairness metrics for the desk-rejection system, we introduce an optimization-based framework to address these metrics. We first study the individual fairness-aware submission limit problem to minimize the individual fairness measure $\zeta_{\mathrm{ind}}$ in Definition~\ref{def:individual_fair}.

\begin{definition}[Individual Fairness-Aware Submission Limit Problem]\label{def:ind_fair_min}
    We consider the following optimization problem:
\begin{align*}
    & ~ \min_{S \subseteq \mathcal{P}} \zeta_{\mathrm{ind}}(S) \\
    \mathrm{s.t.} & ~ |\{p_j \in S : a_i \in A_j\}| \leq x, \quad \forall a_i \in \mathcal{A}.
\end{align*}
\end{definition}

To represent the fairness metric minimization problem in matrix form, we introduce the following definition:

\begin{definition}[Author-Paper Matrix]\label{def:author_paper_mat}
    Let $W \in \{0, 1\}^{n \times m}$ denote the author-paper matrix for the author set $\mathcal{A}$ and paper set $\mathcal{P}$. Then, we define $W_{i,j} = 1$ if author $a_i$ is a coauthor of paper $p_j$, and $W_{i,j} = 0$ otherwise.
\end{definition}

Therefore, we present a more tractable integer
programming form of the original problem and prove its equivalence to the original
formulation:

\begin{definition}[Individual Fairness-Aware Submission Limit Problem, Matrix Form]\label{def:ind_fair_min_matrix}
    We consider the following integer optimization problem:
    \begin{align*}
        & ~ \min_{r \in \{0,1\}^m} \| \mathbf{1}_n - D^{-1}Wr\|_\infty \\
    \mathrm{s.t.} & ~ 
     (W r) / x \leq \mathbf{1}_n
    \end{align*}
    where $D = \Diag(|P_1|, \cdots, |P_n|)$, and the rejection vector $r \in \{0, 1\}^m$ is a 0-1 vector, with $r_j = 1$ indicating that paper $p_j$ is remained, and $r_j = 0$ indicating that it is desk-rejected. 
\end{definition}

\begin{proposition}[Matrix Form Equivalence for $\zeta_{\mathrm{ind}}$, informal version of Proposition~\ref{prop:equiv_individual_append} in Appendix~\ref{sec:fair_proof}]\label{prop:equiv_individual}
    The individual fairness-aware submission limit problem in Definition~\ref{def:ind_fair_min} and the matrix form integer programming problem in Definition~\ref{def:ind_fair_min_matrix} are equivalent.
\end{proposition}

Unfortunately, solving this integer programming problem is highly non-trivial, which means it may not yield a feasible solution within a reasonable time for large-scale conference submission systems. We establish the computational hardness of this problem in the following theorem:
 
\begin{theorem}[Hardness, informal version of Theorem~\ref{thm:indi_nphard_append} in Appendix~\ref{sec:indi_fair_hard_append}]\label{thm:indi_nphard}
    The Individual Fairness-Aware Submission Limit Problem defined in Definition~\ref{def:ind_fair_min} is $\NPhard$.
\end{theorem}

Since minimizing individual fairness is computationally intractable, our fairness-aware desk-rejection system instead focuses on minimizing group fairness.

\subsection{Group Fairness Optimization}\label{sec:fair_optim}

Given the inherent hardness of individual fairness optimization, we address the fairness problem using an alternative yet equally important metric: group fairness, as defined in Definition~\ref{def:group_fair}. This metric is not only a crucial fairness measure in its own right but also serves as a lower bound for individual fairness as stated in Proposition~\ref{lem:fair_metric_ineq}, potentially improving individual fairness implicitly. 

Following a similar approach in Section~\ref{sec:indi_fair_hard}, we first formulate the submission limit problem with respect to group fairness and derive a more tractable integer programming formulation in matrix form:

\begin{definition}[Group Fairness-Aware Submission Limit Problem]\label{def:group_fair_min}
    We consider the following optimization problem:
\begin{align*}
    & ~ \min_{S \subseteq \mathcal{P}} \zeta_{\mathrm{group}}(S) \\
    \mathrm{s.t.} & ~ |\{p_j \in S : a_i \in A_j\}| \leq x, \quad \forall a_i \in \mathcal{A}.
\end{align*}
\end{definition}

\begin{definition}[Group Fairness-Aware Submission Limit Problem, Matrix Form]\label{def:group_fair_min_mat_new}
    We consider the following integer programming problem:
    \begin{align*}
        &~ \max_{r \in \{0, 1\}^m} \mathbf{1}^\top_n D^{-1} W r \\ 
        \mathrm{s.t.} 
        & ~ (W r) / x \leq \mathbf{1}_n,
    \end{align*}
    where $D = \Diag(|P_1|, \cdots, |P_n|)$, and the rejection vector $r \in \{0, 1\}^m$ is a 0-1 vector, with $r_j = 1$ indicating that paper $p_j$ is remained, and $r_j = 0$ indicating that it is desk-rejected. 
\end{definition}

\begin{proposition}[Matrix Form Equivalence for $\zeta_{\mathrm{group}}$, informal version of Proposition~\ref{lem:group_fair_min_equiv_append} in Appendix~\ref{sec:fair_proof}]\label{lem:group_fair_min_equiv}
    The fairness-aware submission limit problem in Definition~\ref{def:group_fair_min} and the matrix form integer programming problem in Definition~\ref{def:group_fair_min_mat_new} are equivalent.
\end{proposition}

However, solving integer programming problems is practically challenging. To this end, we first relax the feasible region of $r$ to $[0,1]^m$, and then analyze the resulting relaxed problem.  

\begin{definition}[Group Fairness-Aware Submission Limit Problem, Relaxation]\label{def:group_fair_min_mat_relax_new}
    We consider the optimization problem
    \begin{align*}
    &~ \max_{r \in [0, 1]^m}  \mathbf{1}^\top_nD^{-1}Wr \\ 
        \mathrm{s.t.} 
        & ~ (Wr)/x\le \mathbf{1}_n,
\end{align*}
where $D = \Diag(|P_1|, \cdots, |P_n|)$, and the rejection vector $r \in \{0, 1\}^m$ is a 0-1 vector, with $r_j = 1$ indicating that paper $p_j$ is remained, and $r_j = 0$ indicating that it is desk-rejected. 
\end{definition}

Fortunately, the relaxed problem is a linear program, which can be efficiently solved using standard linear programming solvers. Moreover, its optimal solution is equivalent to that of the original integer programming problem, an this result is formalized in the following theorem:

\begin{theorem}[Optimal Solution Equivalence of the Relaxed Problem, informal version of Theorem~\ref{thm:lp_equiv_append} in Appendix~\ref{sec:fair_proof}]\label{thm:lp_equiv}
    The optimal solution of the relaxed linear programming problem in Definition~\ref{def:group_fair_min_mat_relax_new} is equivalent to the optimal solution of the original integer programming problem in Definition~\ref{def:group_fair_min_mat_new}.
\end{theorem}

\begin{algorithm}[!ht]
\caption{Fairness-Aware Desk-Reject Algorithm}
\label{alg:fair_desk_reject_algo}
\begin{algorithmic}[1]

\State {\color{blue} /* $\mathcal{A}$ denotes the set of $n$ authors. */}
\State {\color{blue} /* $\mathcal{P}$ denote the set of $m$ papers. */}
\State {\color{blue} /* Author $a_i \in \mathcal{N}$ has a subset of papers $P_i \subset \mathcal{P}$. */}
\State {\color{blue} /* Paper $p_j \in \mathcal{P}$ is coauthored by a subset of authors $A_j \subseteq \mathcal{A}$.*/}
\State {\color{blue} /* $x$ represents the submission limit for each author.*/}

\Procedure{FairDeskReject}{$\mathcal{A}, \mathcal{P}, x$} 
\State {\color{blue} /* Initialize the constants of the problem. */}
\For{$i \in [n], j \in [m]$}
\If{$p_j \in \mathcal{A}_i$}
\State $W_{i,j}\gets 1$
\Else
\State $W_{i,j}\gets 0$
\EndIf
\EndFor
\State $D\gets\Diag(|P_1|, \ldots, |P_n|)$
\State {\color{blue} /* Solve the linear programming problem in Definition~\ref{def:group_fair_min_mat_relax_new}. */}
\State $r^{\star} \gets \mathsf{LPSolver}(W, D, x, r^0)$
\State {\color{blue} /* Transform the solution. */}
\State $S\gets\emptyset$
\For {$j\in[m]$}
\If {$r_j = 1$}
\State $S\gets S \cup \{p_j\}$
\EndIf
\EndFor
\State \Return $S$ 
\EndProcedure
\end{algorithmic}
\end{algorithm}

This theoretical result is significant as we formally establish that the group fairness-aware submission problem in Definition~\ref{def:group_fair_min} reduces to a linear programming (LP) problem with guaranteed optimality, solvable using off-the-shelf LP solvers. We formalize this procedure in Algorithm~\ref{alg:fair_desk_reject_algo}, where $\mathsf{LPSolver}$ denotes any standard LP solver, including but not limited to the simplex method~\cite{bg69}, interior-point path-finding methods~\cite{ls14}, and state-of-the-art stochastic central path methods~\cite{cls19, jswz21}.

\begin{remark}
    The time complexity of our fairness-aware desk-rejection algorithm in Algorithm~\ref{alg:fair_desk_reject_algo} aligns with modern linear programming solvers. For instance, using the stochastic central path method~\cite{cls21,jswz21,vlss20,sy21}, it achieves a time complexity of $O^*(m^{2.37} \log(m/\delta))$, where $\delta$ represents the relative accuracy corresponding to a \((1+\delta)\)-approximation guarantee.
\end{remark}

\begin{remark}
    In practice, major AI conferences routinely process submissions at the scale of $m \sim 10^4$~\cite{stanford_ai_index}. Given this regime, our algorithm guarantees efficient computation, enabling fairness-aware desk rejection within tractable timeframes, even for large-scale conferences.
\end{remark}

\begin{algorithm}[!ht]
\caption{Conventional Desk-Reject Algorithm}
\label{alg:conventional_desk_reject_algo}
\begin{algorithmic}[1]

\Procedure{DeskRejct}{$\mathcal{A}, \mathcal{P}, x$} 
\State {\color{blue} /*  Initialize registered paper set for each author. */}
\For {$i = 1 \to n$}
\State $R_i \gets \emptyset$
\EndFor
\State {\color{blue} /* Initialize the subset of remaining papers. */}
\State $S \gets \mathcal{P}$
\State {\color{blue} /*  Process each paper in submission order.\. */}
\For {$j = 1 \to m$}
\For {$i \in A_j$}
\State {\color{blue} /* If author $a_i$ has reached the submission limit, the paper will be rejected.*/}
\If {$|R_i| \geq x$}
\State $S \gets S \setminus \{p_j\}$
\State \textbf{break}
\EndIf
\EndFor
\State {\color{blue} /* If paper $p_j$ is not rejected, we add it to each co-author's registered paper set.*/}
\If{$p_j \in S$}
\For {$i \in A_j$}
\State $R_i \gets R_i \cup \{ j \}$
\EndFor
\EndIf
\EndFor
\State \Return $S$
\EndProcedure
\end{algorithmic}
\end{algorithm}

\section{Case Study}\label{sec:case_study}

Since desk-rejection data from top AI conferences is not publicly available, and fully open-review conferences like ICLR do not impose submission limits, evaluating real-world conference submissions is impractical. Therefore, we present a case study to demonstrate how our proposed desk-rejection algorithm more effectively addresses fairness issues. Additional case studies are provided in Appendix~\ref{sec:more_case_study}.

Let the paper subscript $j$ in $p_j \in \mathcal{P}$ denote the submission order. We analyze the widely used desk-rejection system (e.g., CVPR 2025) in Algorithm~\ref{alg:conventional_desk_reject_algo}, which rejects all papers submitted after an author’s $x$-th submission. To highlight its limitations, we present a minimal working example:

\begin{example}
    Consider a submission limit problem as defined in Definition~\ref{def:submit_limit_problem} with $n=2$, $x=25$, and $m=26$. Author $a_1$ submits all papers $p_1, \ldots, p_{26}$, while author $a_2$ submits only $p_{26}$.
\end{example}

Given the ideal desk-rejection criteria in Definition~\ref{def:good_solution}, it is evident that we can reject any papers in $\{p_1, p_2, \ldots, p_{25}\}$ following the techniques in Lemma~\ref{lem:n_eq_2_positive}. After rejection, since $a_1$ retains 25 papers and $a_2$ retains 1 paper, the fairness metrics are $\zeta_{\mathrm{ind}}(S) = \max\{1/26, 0\} = 1/26$ and $\zeta_{\mathrm{group}}(S) = \frac{1}{2}(1/26 + 0) = 1/52$.

On the other hand, the CVPR 2025 algorithm, as described in Algorithm~\ref{alg:conventional_desk_reject_algo}, rejects $p_{26}$, retaining $S = \{p_1, \ldots, p_{25}\}$. This unfairly penalizes $a_2$, resulting in $\zeta_{\mathrm{ind}}(S) = \max\{1/26, 1\}= 1$ and $\zeta_{\mathrm{group}}(S) = \frac{1}{2}(1/26 + 1) = 27/52$,
which is much worse compared with the ideal results. In contrast, our method in Algorithm~\ref{alg:fair_desk_reject_algo} solves the linear program and recovers the ideal solution, achieving the same fairness metrics as the optimal case.

A simple workaround to mitigate unfairness in conventional desk-rejection systems is the roulette algorithm, which randomly rejects papers from non-compliant authors like $a_1$ until the submission limit $x$ is reached. However, this heuristic cannot fully prevent the rejection of the undesirable paper $p_{26}$ and results in suboptimal fairness outcomes compared to our fairness-aware rejection, since the expected fairness metrics under the roulette algorithm satisfy $\mathbb{E}[\zeta_{\mathrm{ind}}] = (25/26) \cdot(1/26) + (1/26) \cdot 1 \leq 1/26$ and $\mathbb{E}[\zeta_{\mathrm{group}}] = (25/26) \cdot (1/52) + (1/26) \cdot (27/52) \leq 27/52$.

Thus, this example illustrates that conventional desk-rejection systems in top conferences such as CVPR can suffer from severe fairness issues, whereas our proposed method effectively mitigates these problems. 

Additionally, this example also highlights another noteworthy consequence of the conventional desk-rejection system. Specifically, authors collaborating with senior researchers who have numerous submissions may have to compete for earlier submission slots to avoid desk rejection. However, the submission order should not influence whether a paper is accepted, which reveals the unintended implications of the order-based desk-rejection system.

\section{Conclusion} \label{sec:conclusion}

In this work, we identify the fairness issue in the desk-rejection mechanisms of AI conferences under submission limits.
Our theoretical analysis proves that an ideal system that rejects papers solely based on authors' non-compliance, without unfairly penalizing others due to collective punishment, is impossible. We further consider an optimization-based fairness-aware desk-rejection system to alleviate the unfairness problem. In this system, we considered two fairness metrics: individual fairness and group fairness. We formally established that optimizing individual fairness in desk-rejection is NP-hard, while optimizing group fairness can be reduced to a linear programming problem that can be solved highly efficiently. Through case studies, we showed that the proposed method outperforms the existing desk-rejection system with submission limits in top conferences. For future works, we plan to collaborate with conference organizers to implement and validate our system in real-world settings, contributing to a more equitable AI community. 

\ifdefined\isarxiv
\else
\clearpage
\section*{Impact Statement}
This paper seeks to advance fairness in desk-rejection systems employed by top AI conferences. By formally defining the paper submission limit problem, we demonstrate that an ideal system that rejects papers solely based on excessive submissions without negatively impacting innocent authors is mathematically impossible when multiple authors are involved. To address this, we propose a framework that optimizes desk-rejection fairness based on two different metrics. This ensures that early-career researchers with fewer submissions are less likely to face disproportionate setbacks due to co-authors' submission behavior. While introducing fairness may slightly impact the acceptance rates of senior researchers with a large number of submissions, our approach does not aim to diminish anyone’s research output. Instead, it seeks to balance opportunity across career stages, thereby advancing social justice and contributing to a fairer and more inclusive ML research ecosystem.
\fi

\ifdefined\isarxiv
\bibliographystyle{alpha}
\bibliography{ref} 
\else
\bibliography{ref}

\newcommand{\etalchar}[1]{$^{#1}$}
\begin{thebibliography}{TdSAKKBC18}

\bibitem[AAA{\etalchar{+}}23]{aaa+23}
Josh Achiam, Steven Adler, Sandhini Agarwal, Lama Ahmad, Ilge Akkaya, Florencia~Leoni Aleman, Diogo Almeida, Janko Altenschmidt, Sam Altman, Shyamal Anadkat, et~al.
\newblock Gpt-4 technical report.
\newblock {\em arXiv preprint arXiv:2303.08774}, 2023.

\bibitem[Ahm22]{a22}
Nur Ahmed.
\newblock Scientific labor market and firm-level appropriation strategy in artificial intelligence research.
\newblock Technical report, MIT Sloan Working Paper, 2022.

\bibitem[AMS23]{ams23}
Haris Aziz, Evi Micha, and Nisarg Shah.
\newblock Group fairness in peer review.
\newblock {\em Advances in Neural Information Processing Systems}, 36, 2023.

\bibitem[AMS24]{ams24}
Haris Aziz, Evi Micha, and Nisarg Shah.
\newblock Group fairness in peer review.
\newblock {\em Advances in Neural Information Processing Systems}, 36, 2024.

\bibitem[Ant24]{a24}
Anthropic.
\newblock The claude 3 model family: Opus, sonnet, haiku, 2024.

\bibitem[AS21]{as21}
Ben~W Ansell and David~J Samuels.
\newblock Desk rejecting: A better use of your time.
\newblock {\em PS: Political Science \& Politics}, 54(4):686--689, 2021.

\bibitem[BBM{\etalchar{+}}24]{bbm+24}
Tamay Besiroglu, Sage~Andrus Bergerson, Amelia Michael, Lennart Heim, Xueyun Luo, and Neil Thompson.
\newblock The compute divide in machine learning: A threat to academic contribution and scrutiny?
\newblock {\em arXiv preprint arXiv:2401.02452}, 2024.

\bibitem[BDK{\etalchar{+}}23]{bdk+23}
Andreas Blattmann, Tim Dockhorn, Sumith Kulal, Daniel Mendelevitch, Maciej Kilian, Dominik Lorenz, Yam Levi, Zion English, Vikram Voleti, Adam Letts, et~al.
\newblock Stable video diffusion: Scaling latent video diffusion models to large datasets.
\newblock {\em arXiv preprint arXiv:2311.15127}, 2023.

\bibitem[BDVvdR18]{bvr18}
Thijs Bol, Mathijs De~Vaan, and Arnout van~de Rijt.
\newblock The matthew effect in science funding.
\newblock {\em Proceedings of the National Academy of Sciences}, 115(19):4887--4890, 2018.

\bibitem[Ben20]{b20blog}
Yoshua Bengio.
\newblock Time to rethink the publication process in machine learning.
\newblock \url{https://yoshuabengio.org/2020/02/26/time-to-rethink-the-publication-proce\\ss-in-machine-learning/}, 2020.

\bibitem[BG69]{bg69}
Richard~H Bartels and Gene~H Golub.
\newblock The simplex method of linear programming using lu decomposition.
\newblock {\em Communications of the ACM}, 12(5):266--268, 1969.

\bibitem[CLS19]{cls19}
Michael~B Cohen, Yin~Tat Lee, and Zhao Song.
\newblock Solving linear programs in the current matrix multiplication time.
\newblock In {\em Proceedings of the 51st Annual ACM SIGACT Symposium on Theory of Computing}, pages 938--942, 2019.

\bibitem[CLS21]{cls21}
Michael~B Cohen, Yin~Tat Lee, and Zhao Song.
\newblock Solving linear programs in the current matrix multiplication time.
\newblock {\em Journal of the ACM (JACM)}, 68(1):1--39, 2021.

\bibitem[EDB{\etalchar{+}}22]{edb+22}
Michael~D Ekstrand, Anubrata Das, Robin Burke, Fernando Diaz, et~al.
\newblock Fairness in information access systems.
\newblock {\em Foundations and Trends{\textregistered} in Information Retrieval}, 16(1-2):1--177, 2022.

\bibitem[ER23]{er23}
Faisal~R Elali and Leena~N Rachid.
\newblock Ai-generated research paper fabrication and plagiarism in the scientific community.
\newblock {\em Patterns}, 4(3), 2023.

\bibitem[Fra12]{f12}
Nissim Francez.
\newblock {\em Fairness}.
\newblock Springer Science \& Business Media, 2012.

\bibitem[Gil93]{g93}
Stephen~W Gilliland.
\newblock The perceived fairness of selection systems: An organizational justice perspective.
\newblock {\em Academy of management review}, 18(4):694--734, 1993.

\bibitem[GJ79]{gj79}
Michael~R Garey and David~S Johnson.
\newblock {\em Computers and Intractability: A Guide to the Theory of NP-Completeness}, volume 174.
\newblock freeman San Francisco, 1979.

\bibitem[GLG{\etalchar{+}}21]{glg+21}
Yingqiang Ge, Shuchang Liu, Ruoyuan Gao, Yikun Xian, Yunqi Li, Xiangyu Zhao, Changhua Pei, Fei Sun, Junfeng Ge, Wenwu Ou, et~al.
\newblock Towards long-term fairness in recommendation.
\newblock In {\em Proceedings of the 14th ACM international conference on web search and data mining}, pages 445--453, 2021.

\bibitem[GRB{\etalchar{+}}24]{grb+24}
Isabel~O Gallegos, Ryan~A Rossi, Joe Barrow, Md~Mehrab Tanjim, Sungchul Kim, Franck Dernoncourt, Tong Yu, Ruiyi Zhang, and Nesreen~K Ahmed.
\newblock Bias and fairness in large language models: A survey.
\newblock {\em Computational Linguistics}, pages 1--79, 2024.

\bibitem[GYA22]{gya22}
Yue Guo, Yi~Yang, and Ahmed Abbasi.
\newblock Auto-debias: Debiasing masked language models with automated biased prompts.
\newblock In {\em Proceedings of the 60th Annual Meeting of the Association for Computational Linguistics (Volume 1: Long Papers)}, pages 1012--1023, 2022.

\bibitem[HJA20]{hja20}
Jonathan Ho, Ajay Jain, and Pieter Abbeel.
\newblock Denoising diffusion probabilistic models.
\newblock {\em Advances in neural information processing systems}, 33:6840--6851, 2020.

\bibitem[HSG{\etalchar{+}}22]{hsg+22}
Jonathan Ho, Tim Salimans, Alexey Gritsenko, William Chan, Mohammad Norouzi, and David~J Fleet.
\newblock Video diffusion models.
\newblock {\em Advances in Neural Information Processing Systems}, 35:8633--8646, 2022.

\bibitem[HZRS16]{hzrs16}
Kaiming He, Xiangyu Zhang, Shaoqing Ren, and Jian Sun.
\newblock Deep residual learning for image recognition.
\newblock In {\em Proceedings of the IEEE conference on computer vision and pattern recognition}, pages 770--778, 2016.

\bibitem[JAWD02]{jawd02}
Tom Jefferson, Philip Alderson, Elizabeth Wager, and Frank Davidoff.
\newblock Effects of editorial peer review: a systematic review.
\newblock {\em Jama}, 287(21):2784--2786, 2002.

\bibitem[JSWZ21]{jswz21}
Shunhua Jiang, Zhao Song, Omri Weinstein, and Hengjie Zhang.
\newblock A faster algorithm for solving general lps.
\newblock In {\em Proceedings of the 53rd Annual ACM SIGACT Symposium on Theory of Computing}, pages 823--832, 2021.

\bibitem[Kar72]{k72}
Richard~M Karp.
\newblock Reducibility among combinatorial problems.
\newblock {\em Complexity of Computer Computations}, pages 85--103, 1972.

\bibitem[KC23]{kc23}
Michael~R King and ChatGPT.
\newblock A conversation on artificial intelligence, chatbots, and plagiarism in higher education.
\newblock {\em Cellular and molecular bioengineering}, 16(1):1--2, 2023.

\bibitem[LBNZ{\etalchar{+}}24]{lnz+24}
Kevin Leyton-Brown, Yatin Nandwani, Hedayat Zarkoob, Chris Cameron, Neil Newman, Dinesh Raghu, et~al.
\newblock Matching papers and reviewers at large conferences.
\newblock {\em Artificial Intelligence}, 331:104119, 2024.

\bibitem[LCF{\etalchar{+}}21]{lcf+21}
Yunqi Li, Hanxiong Chen, Zuohui Fu, Yingqiang Ge, and Yongfeng Zhang.
\newblock User-oriented fairness in recommendation.
\newblock In {\em Proceedings of the web conference 2021}, pages 624--632, 2021.

\bibitem[LCX{\etalchar{+}}21]{lcx+21}
Yunqi Li, Hanxiong Chen, Shuyuan Xu, Yingqiang Ge, and Yongfeng Zhang.
\newblock Towards personalized fairness based on causal notion.
\newblock In {\em Proceedings of the 44th International ACM SIGIR Conference on Research and Development in Information Retrieval}, pages 1054--1063, 2021.

\bibitem[Leo13]{l13}
Seth~S Leopold.
\newblock Duplicate submission and dual publication: what is so wrong with them?, 2013.

\bibitem[LS14]{ls14}
Yin~Tat Lee and Aaron Sidford.
\newblock Path finding methods for linear programming: Solving linear programs in o (vrank) iterations and faster algorithms for maximum flow.
\newblock In {\em 2014 IEEE 55th Annual Symposium on Foundations of Computer Science}, pages 424--433. IEEE, 2014.

\bibitem[LY{\etalchar{+}}84]{ly84}
David~G Luenberger, Yinyu Ye, et~al.
\newblock {\em Linear and nonlinear programming}, volume~2.
\newblock Springer, 1984.

\bibitem[MMS{\etalchar{+}}21]{mms+21}
Ninareh Mehrabi, Fred Morstatter, Nripsuta Saxena, Kristina Lerman, and Aram Galstyan.
\newblock A survey on bias and fairness in machine learning.
\newblock {\em ACM computing surveys (CSUR)}, 54(6):1--35, 2021.

\bibitem[MRPC{\etalchar{+}}21]{mrp+21}
Alessandro~B Melchiorre, Navid Rekabsaz, Emilia Parada-Cabaleiro, Stefan Brandl, Oleg Lesota, and Markus Schedl.
\newblock Investigating gender fairness of recommendation algorithms in the music domain.
\newblock {\em Information Processing \& Management}, 58(5):102666, 2021.

\bibitem[RKH{\etalchar{+}}21]{rkh+21}
Alec Radford, Jong~Wook Kim, Chris Hallacy, Aditya Ramesh, Gabriel Goh, Sandhini Agarwal, Girish Sastry, Amanda Askell, Pamela Mishkin, Jack Clark, et~al.
\newblock Learning transferable visual models from natural language supervision.
\newblock In {\em International conference on machine learning}, pages 8748--8763. PMLR, 2021.

\bibitem[SCL23]{scl23}
Ye~Sun, Fabio Caccioli, and Giacomo Livan.
\newblock Ranking mobility and impact inequality in early academic careers.
\newblock {\em Proceedings of the National Academy of Sciences}, 120(34):e2305196120, 2023.

\bibitem[SME20]{sme20}
Jiaming Song, Chenlin Meng, and Stefano Ermon.
\newblock Denoising diffusion implicit models.
\newblock {\em arXiv preprint arXiv:2010.02502}, 2020.

\bibitem[Sta24]{stanford_ai_index}
Stanford.
\newblock Research and development - ai index, 2024.

\bibitem[Sto03]{s03}
WR~Stone.
\newblock Plagiarism, duplicate publication, and duplicate submission: They are all wrong!
\newblock {\em IEEE Antennas and Propagation Magazine}, 45(4):47--49, 2003.

\bibitem[SY21]{sy21}
Zhao Song and Zheng Yu.
\newblock Oblivious sketching-based central path method for linear programming.
\newblock In {\em Proceedings of the 38th International Conference on Machine Learning}, volume 139 of {\em Proceedings of Machine Learning Research}, pages 9835--9847, 2021.

\bibitem[SZK{\etalchar{+}}22]{szk+22}
John Schulman, Barret Zoph, Christina Kim, Jacob Hilton, Jacob Menick, Jiayi Weng, Juan Felipe~Ceron Uribe, Liam Fedus, Luke Metz, Michael Pokorny, et~al.
\newblock Chatgpt: Optimizing language models for dialogue.
\newblock {\em OpenAI blog}, 2(4), 2022.

\bibitem[TdSAKKBC18]{takb18}
Jaime~A Teixeira~da Silva, Aceil Al-Khatib, Vedran Katavi{\'c}, and Helmar Bornemann-Cimenti.
\newblock Establishing sensible and practical guidelines for desk rejections.
\newblock {\em Science and Engineering Ethics}, 24:1347--1365, 2018.

\bibitem[Ten18]{t18}
Jonathan~P Tennant.
\newblock The state of the art in peer review.
\newblock {\em FEMS Microbiology letters}, 365(19):fny204, 2018.

\bibitem[vdBLSS20]{vlss20}
Jan van~den Brand, Yin~Tat Lee, Aaron Sidford, and Zhao Song.
\newblock Solving tall dense linear programs in nearly linear time.
\newblock In {\em Proceedings of the 52nd Annual ACM SIGACT Symposium on Theory of Computing}, pages 775--788, 2020.

\bibitem[VSP{\etalchar{+}}17]{vsp+17}
Ashish Vaswani, Noam Shazeer, Niki Parmar, Jakob Uszkoreit, Llion Jones, Aidan~N Gomez, {\L}ukasz Kaiser, and Illia Polosukhin.
\newblock Attention is all you need.
\newblock {\em Advances in neural information processing systems}, 30, 2017.

\bibitem[WCZ{\etalchar{+}}23]{wcz+23}
Yilin Wang, Zeyuan Chen, Liangjun Zhong, Zheng Ding, Zhizhou Sha, and Zhuowen Tu.
\newblock Dolfin: Diffusion layout transformers without autoencoder.
\newblock {\em arXiv preprint arXiv:2310.16305}, 2023.

\bibitem[WHZ20]{whz20}
Ruotong Wang, F~Maxwell Harper, and Haiyi Zhu.
\newblock Factors influencing perceived fairness in algorithmic decision-making: Algorithm outcomes, development procedures, and individual differences.
\newblock In {\em Proceedings of the 2020 CHI conference on human factors in computing systems}, pages 1--14, 2020.

\bibitem[WJW19]{wjw19}
Yang Wang, Benjamin~F Jones, and Dashun Wang.
\newblock Early-career setback and future career impact.
\newblock {\em Nature communications}, 10(1):4331, 2019.

\bibitem[WMM{\etalchar{+}}22]{wmm+22}
Haolun Wu, Bhaskar Mitra, Chen Ma, Fernando Diaz, and Xue Liu.
\newblock Joint multisided exposure fairness for recommendation.
\newblock In {\em Proceedings of the 45th International ACM SIGIR Conference on research and development in information retrieval}, pages 703--714, 2022.

\bibitem[WSD{\etalchar{+}}24]{wsd+24}
Zirui Wang, Zhizhou Sha, Zheng Ding, Yilin Wang, and Zhuowen Tu.
\newblock Tokencompose: Text-to-image diffusion with token-level supervision.
\newblock In {\em Proceedings of the IEEE/CVF Conference on Computer Vision and Pattern Recognition}, pages 8553--8564, 2024.

\bibitem[WXZ{\etalchar{+}}24]{wxz+24}
Yilin Wang, Haiyang Xu, Xiang Zhang, Zeyuan Chen, Zhizhou Sha, Zirui Wang, and Zhuowen Tu.
\newblock Omnicontrolnet: Dual-stage integration for conditional image generation.
\newblock In {\em Proceedings of the IEEE/CVF Conference on Computer Vision and Pattern Recognition}, pages 7436--7448, 2024.

\bibitem[ZFH{\etalchar{+}}21]{zfh+21}
Yang Zhang, Fuli Feng, Xiangnan He, Tianxin Wei, Chonggang Song, Guohui Ling, and Yongdong Zhang.
\newblock Causal intervention for leveraging popularity bias in recommendation.
\newblock In {\em Proceedings of the 44th international ACM SIGIR conference on research and development in information retrieval}, pages 11--20, 2021.

\end{thebibliography}
\bibliographystyle{icml2025}

\fi

\newpage
\onecolumn
\appendix

\clearpage
\newpage
\begin{center}
    \textbf{\LARGE Appendix}
\end{center}

\paragraph{Roadmap.} In Section~\ref{sec:dilemma_proof}, we supplement the missing proofs in Section~\ref{sec:dr_dilemma}. In Section~\ref{sec:fair_proof}, we present the missing proofs in Section~\ref{sec:fair}. In Section~\ref{sec:more_case_study}, we show additional case studies. In Section~\ref{app:sec:conference_links}, we provide the details related to conference submission limits. 

\section{Missing Proofs in Section 4} \label{sec:dilemma_proof}

In this section, we provide the complete technical proofs for Theorem~\ref{thm:main_res_general} in Section~\ref{sec:dr_dilemma}. 
In Section~\ref{sec:dilemma_proof_defs}, we first introduce key definitions that will be useful 
To structure our analysis, we in the subsequent proofs. We then establish positive results for the cases where $n \leq 2$ in Section~\ref{sec:positive_results}, followed by negative results for $n \geq 3$ in Section~\ref{sec:negative_results}.

\subsection{Basic Definitions}\label{sec:dilemma_proof_defs}

To systematically analyze the desk-rejection problem, we begin by classifying authors based on their submission behavior and their relationship to co-authors. This classification will help us organize and present the proofs in a more structured and readable manner.

\begin{definition}[Author Categories]\label{def:three_kinds_authors}
For any author $a_i \in \mathcal{A}$, we define the following categories:
\begin{itemize}
    \item \textbf{Non-compliant}: An author $a_i$ is non-compliant if they have submitted more than $x$ papers, i.e., $|P_i| > x$. Such authors exceed the submission limit and are subject to desk-rejection under the policy.

    \item \textbf{Vulnerable}: An author $a_i$ is vulnerable if they have submitted no more than $x$ papers ($|P_i| \leq x$) but have at least one non-compliant co-author, i.e., $\exists k \in C_i$ such that $|P_k| > x$. Although these authors comply with the submission limit, they are at risk of being unfairly penalized due to their co-authors' non-compliance.

    \item \textbf{Safe}: An author $a_i$ is safe if they have submitted no more than $x$ papers ($|P_i| \leq x$) and all their co-authors are also compliant, i.e., $\forall k \in C_i$, $|P_k| \leq x$. These authors are guaranteed to retain all their submissions, as neither they nor their co-authors violate the submission limit.
\end{itemize}
\end{definition}

Next, we formalize the notion of achievability for the ideal desk-rejection system.

\begin{definition}[Achievability]\label{def:achievability}
Given a submission limit problem instance as defined in Definition~\ref{def:submit_limit_problem}:
\begin{itemize}
    \item \textbf{Positive result}: A problem instance is a positive result if there exists an algorithm that can achieve the ideal desk-rejection as defined in Definition~\ref{def:good_solution}.
    
    \item \textbf{Negative result}: A problem instance is a negative result if,  under proper conditions, no algorithm can achieve the ideal desk-rejection as defined in Definition~\ref{def:good_solution}.
\end{itemize}
\end{definition}

In the following sections, we will use these definitions to systematically prove the positive results for small numbers of authors ($n \leq 2$) and the negative results for larger numbers of authors ($n \geq 3$), which covers two cases in Theorem~\ref{thm:main_res_general}.

\subsection{Positive Results} \label{sec:positive_results}
In this subsection, we present two positive results that support the $n \leq 2$ case in Theorem~\ref{thm:main_res_general}. We begin with the positive result for $n = 1$ and any $x \in \mathbb{N}_+$.

\begin{lemma}[Positive result for $n = 1$ and any $x \in \mathbb{N}_+$, general case]\label{lem:n_eq_1_positive_general}
    If the following conditions hold:
    \begin{itemize}
        \item Let $n = 1$ denote the number of authors as defined in Definition~\ref{def:submit_limit_problem}.
        \item Let $x \in \mathbb{N}_+$ denote the maximum number of submissions allowed for each author in the conference.
    \end{itemize}
    Then, there exists an algorithm that achieves the ideal desk-rejection as defined in Definition~\ref{def:good_solution}.
\end{lemma}

\begin{proof}
We consider the three cases for the only author $a_1$: non-compliant, vulnerable, and safe, as defined in Definition~\ref{def:three_kinds_authors}.

\paragraph{Case 1: Non-compliant author.} If author $a_1$ is non-compliant, we desk-reject $(|P_1| - x)$ papers. This ensures that exactly $x$ papers remain, satisfying the ideal desk-rejection condition.

\paragraph{Case 2: Vulnerable author.} Since $n = 1$ and there is only one author, author $a_1$ has no co-authors to make itself vulnerable. Therefore, this case cannot happen.

\paragraph{Case 3: Safe author.} If author $a_1$ is safe, no papers need to be rejected. The ideal desk-rejection condition is trivially satisfied.

In all possible cases, we can achieve the ideal desk-rejection. Thus, the proof is finished. 
\end{proof}

To present the positive result for $n=2$ and any $x \in \mathbb{N}_+$, we first discuss a specific case where all authors are non-compliant.

\begin{lemma} [Positive result for $n=2$ and any $x\in\mathbb{N}_+$, non-compliant author only case] \label{lem:n_eq_2_positive}
If the following conditions hold:
\begin{itemize}
    \item Let $n = 1$ denote the number of authors as defined in Definition~\ref{def:submit_limit_problem}.
    \item All the authors are non-compliant authors as defined in Definition~\ref{def:three_kinds_authors}.
    \item Let $x \in \mathbb{N}_+$ denote the maximum number of submissions allowed for each author in the conference. 
\end{itemize}

Then, there exists an algorithm that achieves the ideal desk-rejection as defined in Definition~\ref{def:good_solution}.
\end{lemma}

\begin{proof}
Let $c \in \mathbb{N}$ denote the number of papers co-authored by both author $a_1$ and author $a_2$. For $i \in \{1, 2\}$, let $b_i \in \mathbb{N}$ denote the number of single-authored papers by author $a_i$. 

We then have:
\begin{align*}
    b_1 + c = |P_1|
\end{align*}
and
\begin{align*}
    b_2 + c = |P_2|.
\end{align*}

\paragraph{Case 1: $c \leq x$.} 
In this case, we have $b_1 \geq |P_1| - x$ and $b_2 \geq |P_2| - x$. Since $b_i$ represents the number of single-authored papers by author $a_i$, we can desk reject exactly $(|P_i| - x)$ papers from author $a_i$.

\paragraph{Case 2: $c > x$.} 
Here, we have $b_1 < |P_1| - x$ and $b_2 < |P_2| - x$. We first desk reject all $b_1$ single-authored papers from author $a_1$ and all $b_2$ single-authored papers from author $a_2$. Next, we desk reject $(c - x)$ co-authored papers from both authors. This ensures that the remaining $x$ papers are co-authored by both $a_1$ and $a_2$. Thus, we have successfully rejected exactly $(|P_i| - x)$ papers from each author $a_i$.

By combining the two cases above, the proof is complete.
\end{proof}

With the help of Lemma~\ref{lem:n_eq_2_positive}, we now establish the positive result for $n=2$ and any $x \in \mathbb{N}_+$.

\begin{lemma} [Positive result for $n=2$ and any $x\in\mathbb{N}_+$, general case] \label{lem:n_eq_2_positive_general}
If the following conditions hold:
\begin{itemize}
    \item Let $n = 1$ denote the number of authors as defined in Definition~\ref{def:submit_limit_problem}.
    \item Let $x \in \mathbb{N}_+$ denote the maximum number of submissions allowed for each author in the conference. 
\end{itemize}

Then, there exists an algorithm that achieves the ideal desk-rejection as defined in Definition~\ref{def:good_solution}.
\end{lemma}

\begin{proof}
We consider two authors, $a_1$ and $a_2$. Without loss of generality, we assume that $a_1$ has at least as many papers as $a_2$, i.e., $|P_1| \geq |P_2|$. By exhaustively enumerating all possible compositions of author types (i.e., non-compliant, vulnerable, or safe) for $a_1$ and $a_2$, we observe that the vulnerable-safe composition is impossible. This is because a vulnerable author must co-author at least one paper with a non-compliant author. After excluding this case, we analyze the remaining possible scenarios as follows:

\paragraph{Case 1: Both $a_1$ and $a_2$ are safe authors.} 
In this case, no papers need to be rejected, and the ideal desk-rejection trivially holds.

\paragraph{Case 2: $a_1$ is a non-compliant author and $a_2$ is a safe author.} 
Since rejecting papers from $a_1$ does not affect $a_2$'s submissions, we can simply reject $(|P_1| - x)$ papers from $a_1$ to achieve the ideal desk-rejection.

\paragraph{Case 3: $a_1$ is a non-compliant author and $a_2$ is a vulnerable author.} 
By Definition~\ref{def:three_kinds_authors}, we have $|P_1| > x$ and $|P_2| \le x$. Let $c := |\{p_j \in S: p_j \in P_1, p_j \in P_2\}|$ denote the number of co-authored papers by $a_1$ and $a_2$. From basic set theory, we know that $c \leq |P_2|$. Since $|P_2| \le x$, it follows that $c \le x$. Therefore, we have:
\begin{align*}
    \underbrace{|P_1| - c}_{\text{Individual papers of } a_1} \ge \underbrace{|P_1| - x}_{\text{Excess papers of } a_1},
\end{align*}
which implies that the number of individual papers authored solely by $a_1$ exceeds the number of over-limit papers for $a_1$. Thus, we can first reject $a_1$'s individual papers without affecting $a_2$'s submissions, thereby achieving the desired ideal desk-rejection.

\paragraph{Case 4: Both $a_1$ and $a_2$ are non-compliant authors.} 
This case directly follows from Lemma~\ref{lem:n_eq_2_positive}.

Combining all the cases above, we conclude that the ideal desk-rejection can always be achieved, which finishes the proof.
\end{proof}

\subsection{Negative Results} \label{sec:negative_results}
In this subsection, we present two positive results that support the $n \ge 3$ case in Theorem~\ref{thm:main_res_general}. We commence by showing the negative result for $n = 3$ and $x=1$.

\begin{lemma}[Negative result for $n=3$ and $x=1$] \label{lem:n_eq_3_negative}
If the following conditions hold:
\begin{itemize}
    \item Let $n = 3$ denote the number of authors as defined in Definition~\ref{def:submit_limit_problem}.
    \item Let $x=1$ denote the maximum number of submissions allowed for each author in the conference. 
\end{itemize}

Then, under proper conditions, no algorithm can achieve the ideal desk-rejection as defined
in Definition~\ref{def:good_solution}.
\end{lemma}

\begin{proof}
Let all the authors be non-compliant authors as defined in Definition~\ref{def:three_kinds_authors}, and let the number of papers be $m=3$. We suppose the three papers $p_1$, $p_2$, and $p_3$ have the following authorship:

\begin{itemize}
    \item Paper $p_1$ is co-authored by $a_1$ and $a_2$.
    \item Paper $p_2$ is co-authored by $a_1$ and $a_3$.
    \item Paper $p_3$ is co-authored by $a_2$ and $a_3$.
\end{itemize}

From the authors' perspective, the relationships are as follows:
\begin{itemize}
    \item Author $a_1$ has papers $p_1$ and $p_2$.
    \item Author $a_2$ has papers $p_1$ and $p_3$.
    \item Author $a_3$ has papers $p_2$ and $p_3$.
\end{itemize}

We enumerate all possible rejection plans and their outcomes in Table~\ref{tab:all_possible_rejections}.

\begin{table}[!ht]
\caption{Remaining number of papers for each author after desk rejection.}
\label{tab:all_possible_rejections}
\begin{center}
\begin{tabular}{|c|c|c|c|}
 \hline
 Rejected Papers & Author $a_1$ & Author $a_2$ & Author $a_3$ \\ \hline
 N/A             & 2            & 2            & 2            \\ \hline
 $p_1$           & 1            & 1            & 2            \\ \hline
 $p_2$           & 1            & 2            & 1            \\ \hline
 $p_3$           & 2            & 1            & 1            \\ \hline
 $p_1, p_2$      & 0            & 1            & 1            \\ \hline
 $p_1, p_3$      & 1            & 0            & 1            \\ \hline
 $p_2, p_3$      & 1            & 1            & 0            \\ \hline
 $p_1, p_2, p_3$ & 0            & 0            & 0            \\ \hline
\end{tabular}
\end{center}
\end{table}

First, suppose we desk reject paper $p_3$. Then, authors $a_2$ and $a_3$ each have one paper remaining, but author $a_1$ still has two papers. To satisfy the constraint $x=1$, we must reject one of $p_1$ or $p_2$.

If we reject $p_1$, author $a_2$ is left with no papers, which is unfair. If we reject $p_2$, author $a_3$ is left with no papers, which is also unfair.

Thus, no rejection plan satisfies the ideal desk rejection condition for all authors. This completes the proof.
\end{proof}

Next, we present the negative result for any $n \geq 3$ and $x = n-2$.

\begin{lemma}[Negative result for any $n \geq 3$ and $x = n-2$] \label{lem:n_geq_3_negative}
If the following conditions hold:
\begin{itemize}
    \item Let $n \ge 3$ denote the number of authors as defined in Definition~\ref{def:submit_limit_problem}.
    \item Let $x=n-2$ denote the maximum number of submissions allowed for each author in the conference. 
\end{itemize}

Then, under proper conditions, no algorithm can achieve the ideal desk-rejection as defined
in Definition~\ref{def:good_solution}.
\end{lemma}

\begin{proof}
In this negative problem instance, we choose the number of papers to be the same as the number of authors, i.e., $m=n$, and we assume all the $n$ authors are non-compliant authors as defined in Definition~\ref{def:three_kinds_authors}.

For each of the $n$ papers $p_i \in \mathcal{P}$, we let $i$-th paper $p_i$ contain $n-1$ authors, excluding only the $i$-th author $a_i$. Specifically, we have:  
\begin{itemize}
    \item The first paper $p_1$ has authors $a_2, a_3, \cdots, a_n$. 
    \item The second paper $p_2$ has authors $a_1, a_3, a_4, \cdots, a_n$.
    \item $\cdots\cdots$
    \item The $(n-1)$-th paper has authors $a_1, a_2, \cdots , a_{n-2}, a_n$.
    \item The $n$-th paper has authors $a_1, a_2, \cdots , a_{n-2}, a_{n-1}$.
\end{itemize}

Since each author is allowed to submit at most $x=n-2$ papers, we must desk-reject at least two papers. We analyze the process of desk-rejecting these two papers step by step.

\textbf{Step 1: Desk-reject the first paper.}

Without loss of generality, we consider rejecting paper $p_1$ first. After this operation, authors $a_2, a_3, \cdots a_n$, will have $n-2$ submitted papers, while author $a_1$ will have $n-1$ submitted papers.

\textbf{Step 2: Desk-reject the second paper.}

Without loss of generality, we consider rejecting paper $p_2$ next. After this operation, authors $a_3, a_4, \cdots a_n$, will have $n-3$ submitted papers, while author $a_1$ and $a_2$ will have $n-2$ submitted papers. 

At this point, it is impossible for authors $a_3, a_4, a_5 \cdots , a_n$ to have exactly $(n-2)$ submitted papers. Therefore, no algorithm can achieve the ideal desk-rejection under the given conditions. This completes the proof.
\end{proof}

\section{Missing Proofs in Section 5}\label{sec:fair_proof}
In this section, we first present the missing proofs for fairness metrics in Section~\ref{sec:fair_metric_append}, and then present the supplementary proofs for the hardness of individual fairness optimization in Section~\ref{sec:indi_fair_hard_append}. Finally, we show the additional proofs for the group fairness optimization problem in Section~\ref{sec:fair_optim_append}.

\subsection{Fairness Metrics}\label{sec:fair_metric_append}
We present the relationship between the fairness metrics.  

\begin{proposition}[Relationship of Fairness Metrics, formal version of Proposition~\ref{lem:fair_metric_ineq} in Section~\ref{sec:fair_metric}]\label{lem:fair_metric_ineq_append}
    For any solution $S\subseteq \mathcal{P}$ for the submission limit problem in Definition~\ref{def:submit_limit_problem}, we have 
    \begin{align*}
        \zeta_{\mathrm{group}}(S) \leq \zeta_{\mathrm{ind}}(S).
    \end{align*}
\end{proposition}
\begin{proof}
    By Definition~\ref{def:group_fair}, we have:
    \begin{align*}
        \zeta_{\mathrm{group}}(S) &=~ \frac{1}{n}\sum_{i \in [n]} c(a_i,S) \\
        &\leq~ \frac{1}{n}\sum_{i \in [n]} \max_{i\in[n]}c(a_i, S) \\ 
        &=~\frac{1}{n}\cdot n \cdot \max_{i\in[n]}c(a_i, S) \\ 
        &=~ \zeta_{\mathrm{ind}}(S),
    \end{align*}
where the first equality directly follows from Definition~\ref{def:group_fair}, the second and the third inequality follow from basic algebra, and the last equality follows from Definition~\ref{def:individual_fair}. Thus, we complete the proof.
\end{proof}

\subsection{Hardness of Individual Fairness-Aware Submission Limit Problem}\label{sec:indi_fair_hard_append}

Before proving the theoretical results in Section~\ref{sec:indi_fair_hard}, we first introduce a useful fact that serves as a foundation for the subsequent proofs.   

\begin{fact}\label{fact:author_paper_count}
For each author $a_i\in\mathcal{A}$, the number of papers after desk-rejection (i.e., $|\{p_j \in  S : a_i \in A_j\}|$) can be written as $W_i^\top r$.
\end{fact}
\begin{proof}
    This simply follows from:
    \begin{align*}
        W^\top_ir &=~ \sum_{j\in[m]} W_{i,j}\cdot r_j \\ 
        &=~ |\{j\in[m]:W_{i, j}=1,r_j=1\}| \\ 
        &=~ |\{p_j\in\mathcal{P}:a_i\in A_j, p_j\in S\}| \\ 
        &=~|\{p_j \in  S : a_i \in A_j\}|,
    \end{align*}
where the first and the second equality follow from basic algebra and set theory, and the third and the fourth equality follow from Definition~\ref{def:submit_limit_problem}.  
\end{proof}

With the help of the aforementioned fact, we now prove the equivalence of the matrix form for the individual fairness problem.

\begin{proposition}[Matrix Form Equivalence for $\zeta_{\mathrm{ind}}$, formal version of Proposition~\ref{prop:equiv_individual} in Section~\ref{sec:indi_fair_hard}]\label{prop:equiv_individual_append}
    The individual fairness-aware submission limit problem in Definition~\ref{def:ind_fair_min} and the matrix form integer programming problem in Definition~\ref{def:ind_fair_min_matrix} are equivalent.
\end{proposition}
\begin{proof}
    In Definition~\ref{def:ind_fair_min}, the paper set $\mathcal{P}$ consists of $m$ papers, each of which can either be maintained or desk-rejected. Thus, the subset of maintained papers, $\mathcal{S}$, can be represented by a 0-1 vector $r \in \{0, 1\}^m$, where $r_j = 1$ indicates that paper $p_j$ is maintained, and $r_j = 0$ indicates that it is desk-rejected. We now establish the equivalence of both the objective function and the constraints in these two formulations.

    \paragraph{Part 1: Optimization Objective.} We first consider the objective function $\mathbf{1}^\top_nD^{-1}Wr$ in Definition~\ref{def:ind_fair_min_matrix}:
    \begin{align*}
        \min_{r \in \{0,1\}^m} \| \mathbf{1}_n - D^{-1}Wr\|_\infty &=~ \min_{r \in \{0,1\}^m} \max_{i\in[n]} (1 - (D^{-1}Wr)_i) \\ 
        &=~ \min_{r \in \{0,1\}^m} \max_{i\in[n]} (1 - (W_i^\top r)_i/D_{i,i}) \\
        &=~ \min_{r \in \{0,1\}^m} \max_{i\in[n]} (1 - (W_i^\top r)_i/|P_i|) \\ 
        &=~ \min_{r \in \{0,1\}^m} \max_{i\in[n]} (1 - |\{p_j \in  S : a_i \in A_j\}|/|P_i|) \\ 
        &=~ \min_{r \in \{0,1\}^m} \max_{i\in[n]}c(a_i, S) \\ 
        &=~ \min_{r \in \{0,1\}^m} \zeta_{\mathrm{ind}}(S),
    \end{align*}
where the first equality follows from the definition of infinity norm, the second equality follows from basic algebra, the third equality follows from Definition~\ref{def:ind_fair_min_matrix}, the fourth equality follows from Fact~\ref{fact:author_paper_count}, the fifth equality follows from Definition~\ref{def:cost}, and the last equality follows from Definition~\ref{def:individual_fair}. By decoding $r$ back into the paper subset $S$, we recover the original optimization objective in Definition~\ref{def:ind_fair_min}.

\paragraph{Part 2: Constraints.} The constraint in Definition~\ref{def:ind_fair_min_matrix} can be rewritten using basic algebra as:
     \begin{align*}
         W_i \cdot r \leq x, \quad \forall i \in [n].
     \end{align*}
     By applying Fact~\ref{fact:author_paper_count}, we see that this constraint is equivalent to its counterpart in Definition~\ref{def:ind_fair_min}.

Since both the objective function and constraints in Definition~\ref{def:ind_fair_min} and Definition~\ref{def:ind_fair_min_matrix} are equivalent, the proof is complete.
\end{proof}

To show the hardness of the individual fairness problem, we first present a classical set cover problem with well-established hardness. 

\begin{definition}[Set Cover Problem \cite{k72,gj79}]\label{def:set_cover}
The Set Cover problem is the following: 
\begin{itemize}
    \item {\bf Input:} A universe $U = \{1, \ldots, n\}$, a family of sets 
    $\{S_1, \ldots, S_m\} \subseteq 2^U$, and a integer $K > 0$. 
    \item {\bf Question:} Is there a subfamily $\{S_j : j \in J\}$ for some 
    $J \subseteq \{1,\ldots,m\}$ and $|J| \leq K$ that covers $U$, i.e., $\bigcup_{j \in J} S_j = U$?
\end{itemize}
\end{definition}

\begin{lemma}[Folklore \cite{k72,gj79}]\label{lem:set_cover_np_hard}
    The Set Cover problem defined in Definition~\ref{def:set_cover} is $\NPhard$.
\end{lemma}

Additionally, we also present a technical lemma which is useful for showing the hardness of the individual fairness problem. 

\begin{lemma}\label{lem:useful_lemma}
    For any $r \in \{0,1\}^m$, the following two statements are equivalent:
    \begin{itemize}
       \item {\bf Part 1.} $\|\mathbf{1}_n - D^{-1}Wr\|_\infty \leq 1 - \frac{1}{\min_{i \in [n]}|P_i|}$.
       \item {\bf Part 2.} $\min_{i\in[n]} (Wr)_i \geq 1$.
    \end{itemize}
\end{lemma}
\begin{proof}
    We first show that Part 1 implies Part 2.
    Suppose that
    \begin{align*}
        \|\mathbf{1}_n - D^{-1}Wr\|_\infty \leq 1 - \frac{1}{\min_{i \in [n]}|P_i|}.
    \end{align*}
    By the definition of the infinity norm, we have
    \begin{align*}
         1- \frac{(Wr)_{i'}}{|P_{i'}|} \leq 1 - \frac{1}{\min_{i \in [n]}|P_i|}, \quad \forall i' \in [n].
    \end{align*}
    Rearranging gives
    \begin{align*}
        (Wr)_{i'} \geq \frac{|P_{i'}|}{\min_{i \in [n]}|P_i|} \geq 1, \quad \forall i' \in [n].
    \end{align*}
    Since for all $i'\in[n]$, we have $(Wr)_{i'} \geq 1$, we can conclude that $\min_{i\in[n]} (Wr)_i \geq 1$.

    Now we show that that Part 2 implies Part 1.
    Suppose that $\min_{i\in[n]} (Wr)_i \geq 1$,  then we have $(Wr)_{i} \geq 1$ for all $i \in [n]$, which implies that for all $i \in [n]$, 
    \begin{align*}
        1 - \frac{(Wr)_i}{|P_i|} \leq  1 - \frac{1}{|P_i|} \leq  1 - \frac{1}{\max_{i' \in [n]} |P_{i'}|}.
    \end{align*}
    Hence
    \begin{align*}
        \|\mathbf{1}_n - D^{-1}Wr\|_\infty \leq 1 - \frac{1}{\min_{i \in [n]}|P_i|}.
    \end{align*}
    Thus the proof is complete.
\end{proof}

\begin{theorem}[Hardness, formal version of Theorem~\ref{thm:indi_nphard} in Section~\ref{sec:indi_fair_hard}]\label{thm:indi_nphard_append}
    The Individual Fairness-Aware Submission Limit Problem defined in Definition~\ref{def:ind_fair_min} is $\NPhard$.
\end{theorem}
\begin{proof}
    By Proposition~\ref{prop:equiv_individual}, it sufficies to reduce Set Cover problem to the integer optimization problem of the matrix form in Definition~\ref{def:ind_fair_min_matrix}.

    Given an instance of Set Cover, we build the matrix $W \in \{0, 1\}^{n\times m}$ by defining $W_{i,j} = 1$ if element $i \in S_j$, and 0 otherwise.
    Now set $|P_i| = \sum_{j \in [m]} W_{i,j}$ for every row $i \in [n]$. Finally, we choose $x = m$. We reduce the Set Cover problem to the following optimization problem:
    \begin{align*}
        & ~ \min_{r \in \{0,1\}^m} \| \mathbf{1}_n - D^{-1}Wr\|_\infty \\
    \mathrm{s.t.}
    & ~ W r \leq m \mathbf{1}_n, \\
    & ~ \|r\|_1 \leq K.
    \end{align*}
    Note that this problem is easier than the optimization problem defined in Definition~\ref{def:ind_fair_min}. The constraint $ W r \leq m \mathbf{1}_n$ is always satisfied, so we can drop it out. Now, it suffices to consider the decision problem:
    \begin{align*}
        &~ \mathrm{Find~} r \in \{0,1\}^m \\ \mathrm{s.t.~} &~
       \|\mathbf{1}_n - D^{-1}Wr\|_\infty \leq 1 - \frac{1}{\min_{i \in [n]}|P_i|}, \\
       &~ \|r\|_1 \leq K.
    \end{align*}

    Note that $\|\mathbf{1}_n - D^{-1}Wr\|_\infty \leq 1 - \frac{1}{\min_{i \in [n]}|P_i|}$ is equivalent to 
       $\min_{i\in[n]} (Wr)_i \geq 1$ by Lemma~\ref{lem:useful_lemma}.

    Hence the problem is equivalent to
    \begin{align*}
       \mathrm{Find~} r \in \{0,1\}^m  \mathrm{~~~s.t.~~~}
       \min_{i\in[n]} (Wr)_i \geq 1 \mathrm{~~and~~} 
        \|r\|_1 \leq K.
    \end{align*}
    It is not hard to see that the Set Cover problem has a solution if and only if the above problem has a solution.
    Requiring $\min_{i \in [n]} (W r)_i > 1$ exactly means that each element $i$ in the universe is covered by at least set $S_j$. The constraint $\|r\|_1 \leq K$ means that the size of cover is at most $K$. In other words, there exists a subfamily of size at most $K$ covering all elements if and only if there is an $r \in \{0,1\}^m$ with $\min_{i \in [n]} (W r)_i > 1$ and $\|r\|_1 \leq K$.
    
    Therefore, by Lemma~\ref{lem:set_cover_np_hard}, the individual fairness-aware submission limit problem is $\NPhard$.
\end{proof}

\subsection{Group Fairness Optimization}\label{sec:fair_optim_append}

Now, we present the missing proofs on both matrix form equivalence and linear programming optimal solution equivalence for the group fairness optimization problem. 

\begin{proposition}[Matrix Form Equivalence for $\zeta_{\mathrm{group}}$, formal version of Proposition~\ref{lem:group_fair_min_equiv} in Section~\ref{sec:fair_optim}]\label{lem:group_fair_min_equiv_append}
    The problem in Definition~\ref{def:group_fair_min} and the problem in Definition~\ref{def:group_fair_min_mat_new} are equivalent.
\end{proposition}
\begin{proof}
     In Definition~\ref{def:group_fair_min}, there are $m$ papers in $\mathcal{P}$, where each paper can either be maintained or rejected. Thus, we can encode the paper subset $S$ using a binary vector $r \in \{0, 1\}^m$, where $r_j = 1$ indicates that paper $p_j$ is maintained, and $r_j = 0$ indicates that it is desk-rejected. We now demonstrate that both the objective function and the constraints are equivalent.

     \paragraph{Part 1: Optimization Objective.} We first examine the objective function $\mathbf{1}^\top_nD^{-1}Wr$ in Definition~\ref{def:group_fair_min_mat_new}:
     \begin{align*}
         \mathbf{1}^\top_nD^{-1}Wr &=~ \sum_{i\in[n]}(D^{-1}Wr)_i \\ 
         &=~\sum_{i\in[n]}(W\cdot r)_i / |P_i| \\
         &=~\sum_{i\in[n]}(W_i^\top\cdot r) / |P_i| \\
         &=~\sum_{i\in[n]}\frac{|\{p_j \in S: a_i \in A_j\}|}{|P_i|} \\ 
         &=~ \sum_{i\in[n]}(1-c(a_i,S)),
     \end{align*}
     where the first equality follows from basic algebra, the second follows from Definition~\ref{def:group_fair_min_mat_new}, the third follows from matrix-vector multiplication, the fourth follows from Fact~\ref{fact:author_paper_count}, and the final equality follows from Definition~\ref{def:cost}. Consequently, the maximization problem in Definition~\ref{def:group_fair_min_mat_new} can be rewritten as:
     \begin{align*}
        \max_{r \in \{0, 1\}^m} \sum_{i\in[n]}(1-c(a_i,S)).
     \end{align*}
     Since maximizing this objective is equivalent to minimizing $\sum_{i\in[n]} c(a_i,S)$, we can reformulate it as:
     \begin{align*}
         \min_{r \in \{0, 1\}^m} \sum_{i\in[n]} c(a_i,S).
     \end{align*}
     By decoding $r$ back into the paper subset $S$, we recover the original optimization objective in Definition~\ref{def:group_fair_min}. 

     \paragraph{Part 2: Constraints.} Since the constraint is identical to that in the individual fairness minimization problem in Definition~\ref{def:ind_fair_min}, this result follows directly from Part 2 in the proof of Proposition~\ref{prop:equiv_individual_append}.

     Since both the objective function and constraints in Definition~\ref{def:group_fair_min} and Definition~\ref{def:group_fair_min_mat_new} are equivalent, the proof is complete.
\end{proof}

\begin{theorem}[Optimal Solution Equivalence of the Relaxed Problem, formal version of Theorem~\ref{thm:lp_equiv} in Section~\ref{sec:fair_optim}]\label{thm:lp_equiv_append}
    The optimal solution of the relaxed problem in Definition~\ref{def:group_fair_min_mat_relax_new} is equivalent to the optimal solution of the original problem in Definition~\ref{def:group_fair_min_mat_new}.
\end{theorem}
\begin{proof} 
    The problem in Definition~\ref{def:group_fair_min_mat_relax_new} is a linear program since it has a linear objective function $\mathbf{1}^\top_nD^{-1}Wr$ and linear constraints: the box constraint $r\in[0,1]^m$ and a linear inequality constraint $(Wr)/x \leq \mathbf{1}_n$.

    Furthermore, the problem is convex because the objective function is linear, the constraint $(Wr)/x \leq \mathbf{1}_n$ is affine, and the feasible region defined by $r\in[0,1]^m$ is a convex set.

    By the fundamental theorem of linear programming (see Page 23 of~\cite{ly84}), the optimal solution must occur at an extreme point of the convex polytope defined by the constraints. This implies that for all $i \in [m]$, we must have either $r_i = 0$ or $r_i = 1$. Therefore, the optimal solution of the relaxed linear program coincides with that of the original integer program, which finishes the proof.
\end{proof}

\section{Additional Case Studies}\label{sec:more_case_study}

As discussed in Section~\ref{sec:good_solution_hard}, optimizing the individual fairness metric is computationally challenging. Therefore, we minimize the group fairness metric, which serves as a lower bound for individual fairness, as a practical alternative. In this subsection, we present case studies demonstrating the relationship between both types of fairness metrics. 

\begin{example}
Consider a submission limit problem as defined in Definition~\ref{def:submit_limit_problem} with $x = 2$, $n = 3$ authors, and $m = 6$ papers. Let author $a_1$ submit four papers $p_1, p_2, p_3, p_4$, author $a_2$ submit two papers $p_3, p_5$, and author $a_3$ submit two papers $p_4, p_6$. 
\end{example}

In this case, the ideal desk-rejection criteria in Definition~\ref{def:good_solution} reject $p_1$ and $p_2$ (i.e., $S = \{p_3, p_4, p_5, p_6\}$), yielding fairness metrics $\zeta_{\mathrm{ind}}(S) = \max\{1/2, 0, 0\} = 1/2$ and $\zeta_{\mathrm{group}}(S) = \frac{1}{3}(1/2 + 0 + 0) = 1/6$. By applying an LP solver to minimize group fairness using Algorithm~\ref{alg:fair_desk_reject_algo} and enumerating all rejection strategies to verify individual fairness minimization, we observe that minimizing group fairness in this case aligns with minimizing individual fairness as defined in Definition~\ref{def:ind_fair_min_matrix}. This case illustrates that minimizing group fairness can sometimes benefit individual fairness.

However, group fairness and individual fairness are not always consistent. In some cases, prioritizing group fairness may disproportionately burden certain individuals. To illustrate this, we consider the following example.

\begin{example}
    Consider a submission limit problem as defined in Definition~\ref{def:submit_limit_problem} with $x = 2$, $n = 5$ authors, and $m = 4$ papers. Let author $a_1$ submit four papers $p_1, p_2, p_3, p_4$, author $a_2$ submit two papers $p_1, p_2$, and authors $a_3, a_4, a_5$ be coauthors of papers $p_3, p_4$. 
\end{example}

In this scenario, an ideal desk-rejection is impossible because $a_1$ must have two papers rejected, but rejecting any papers would cause at least one of the authors in $a_2, \ldots, a_5$ to fall below the submission limit of $x=2$. Here, group fairness and individual fairness diverge: Algorithm~\ref{alg:fair_desk_reject_algo} minimizes group fairness by rejecting $p_1$ and $p_2$ (i.e., $S = \{p_3, p_4\}$), which unfairly excludes all of $a_2$'s submissions. This results in fairness metrics $\zeta_{\mathrm{group}}(S) = \frac{1}{4}(1/2 + 1 + 0 + 0) = 3/8$ and $\zeta_{\mathrm{ind}}(S) = \max\{1/2, 1, 0, 0\} = 1$. 

Conversely, the individual fairness minimization problem in Definition~\ref{def:ind_fair_min_matrix} rejects one paper from $a_1, a_2$ and another from $a_3, a_4$, leading to $\zeta_{\mathrm{group}}(S) = \frac{1}{4}(1/2 + 1/2 + 1/2 + 1/2) = 1/2$ and $\zeta_{\mathrm{ind}}(S) = \max\{1/2, 1/2, 1/2, 1/2\} = 1/2$.

This example highlights an unintended consequence of minimizing group fairness: it may unfairly penalize authors with fewer coauthors, as rejecting their papers incurs a smaller total cost. On the other hand, optimizing individual fairness inevitably spreads rejections across a broader set of authors, potentially leading to a higher overall fairness cost. Balancing individual and group fairness remains an open challenge, which we leave for future work.

\section{Summary of Conference Links} \label{app:sec:conference_links}

In the introduction, Table~\ref{tab:conference_submission_limit} only gives a brief summary of the conference year and its limitation of per-author submission. 
Thus, we provide a detailed list of conferences in each year in this section, and then summarize the submission limits in Table.~\ref{tab:conference_submission_limit_full}. 
\begin{itemize}
    \item CVPR 
    \begin{itemize}
        \item 2025, \url{https://cvpr.thecvf.com/Conferences/2025/CVPRChanges} 
        \item 2024, \url{https://cvpr.thecvf.com/Conferences/2024/AuthorGuidelines}
    \end{itemize}

    \item ICCV
    \begin{itemize}
        \item 2025, \url{https://iccv.thecvf.com/Conferences/2025/AuthorGuidelines}
        \item 2023, \url{https://iccv2023.thecvf.com/policies-361500-2-20-15.php}
    \end{itemize}
    
    \item AAAI
    \begin{itemize}
        \item 2025, \url{https://aaai.org/conference/aaai/aaai-25/submission-instructions/} 
        \item 2024, \url{https://aaai.org/aaai-24-conference/submission-instructions/} 
        \item 2023, \url{https://aaai-23.aaai.org/submission-guidelines/} 
        \item 2022, \url{https://aaai.org/conference/aaai/aaai-22/}
    \end{itemize}
    
    \item WSDM 
    \begin{itemize}
        \item 2025, \url{https://www.wsdm-conference.org/2025/call-for-papers/} 
        \item 2024, \url{https://www.wsdm-conference.org/2024/call-for-papers/} 
        \item 2023, \url{https://www.wsdm-conference.org/2023/calls/call-papers/} 
        \item 2022, \url{https://www.wsdm-conference.org/2022/calls/} 
        \item 2021, \url{https://www.wsdm-conference.org/2021/call-for-papers.php} 
        \item 2020, \url{https://www.wsdm-conference.org/2020/call-for-papers.php} 
    \end{itemize}
    
    \item IJCAI
    \begin{itemize}
        \item 2025, \url{https://2025.ijcai.org/call-for-papers-main-track/} 
        \item 2024, \url{https://ijcai24.org/call-for-papers/}
        \item 2023, \url{https://ijcai-23.org/call-for-papers/}
        \item 2022, \url{https://ijcai-22.org/calls-papers}
        \item 2021, \url{https://ijcai-21.org/cfp/index.html}
        \item 2020, \url{https://ijcai20.org/call-for-papers/index.html}
        \item 2019, \url{https://www.ijcai19.org/call-for-papers.html}
        \item 2018, \url{https://www.ijcai-18.org/cfp/index.html}
        \item 2017, \url{https://ijcai-17.org/MainTrackCFP.html}
    \end{itemize}
    
    \item KDD 
    \begin{itemize}
        \item 2025, \url{https://kdd2025.kdd.org/research-track-call-for-papers/} 
        \item 2024, \url{https://kdd2024.kdd.org/research-track-call-for-papers/}
        \item 2023, \url{https://kdd.org/kdd2023/call-for-research-track-papers/index.html}
    \end{itemize}
\end{itemize}

\newpage
\begin{table}[!ht]\caption{ 
In this table, we summarize the submission limits of top conferences in recent years. For details of each conference website, we refer the readers to Section~\ref{app:sec:conference_links} in Appendix. 
}  \label{tab:conference_submission_limit_full}
\begin{center}
\begin{tabular}{ |c|c|c|c|c| } 
 \hline
 {\bf Conference Name} & {\bf Year} & {\bf Upper Bound} \\ \hline
 CVPR & 2025 & 25 \\ \hline
 CVPR & 2024 & N/A \\ \hline
 ICCV & 2025 & 25 \\ \hline
 ICCV & 2023 & N/A \\ \hline
 AAAI & 2025 & 10 \\ \hline
 AAAI & 2024 & 10 \\ \hline
 AAAI & 2023 & 10 \\ \hline
 AAAI & 2022 & N/A \\ \hline
 WSDM & 2025 & 10 \\ \hline
 WSDM & 2024 & 10 \\ \hline
 WSDM & 2023 & 10 \\ \hline
 WSDM & 2022 & 10 \\ \hline
 WSDM & 2021 & 10 \\ \hline
 WSDM & 2020 & N/A \\ \hline
 IJCAI & 2025 & 8 \\ \hline
 IJCAI & 2024 & 8 \\ \hline
 IJCAI & 2023 & 8 \\ \hline
 IJCAI & 2022 & 8 \\ \hline
 IJCAI & 2021 & 8 \\ \hline
 IJCAI & 2020 & 6 \\ \hline
 IJCAI & 2019 & 10 \\ \hline
 IJCAI & 2018 & 10 \\ \hline
 IJCAI & 2017 & N/A \\ \hline
 KDD & 2025 & 7 \\ \hline
 KDD & 2024 & 7 \\ \hline
 KDD & 2023 & N/A \\ \hline
\end{tabular}
\end{center}
\end{table}



\end{document}